%% file: main.tex
\documentclass{article}

\usepackage{comment}
\usepackage{microtype}
\usepackage{graphicx}
\usepackage{subfigure}
\usepackage{booktabs} 
\usepackage{hyperref}
\usepackage{xspace} 

\input{math_commands.tex}

\usepackage[accepted]{icml2021}

\newcommand{\RAAV}{RA2.1-Q\xspace}  
\newcommand{\RAAA}{RA3-Q\xspace}  
\newcommand{\RAA}{RA2-Q\xspace}  
\newcommand{\RAAM}{RAM-Q\xspace}

\begin{document}
\setlength{\abovedisplayskip}{0pt}
\setlength{\belowdisplayskip}{0pt}
\setlength{\abovedisplayshortskip}{0pt}
\setlength{\belowdisplayshortskip}{0pt}

\icmltitlerunning{Robust Risk-Sensitive Reinforcement Learning Agents for Trading Markets}

\twocolumn[
\icmltitle{Robust Risk-Sensitive Reinforcement Learning Agents for Trading Markets}



\icmlsetsymbol{intern}{*}

\begin{icmlauthorlist}
\icmlauthor{Yue Gao}{to,borealis,intern}
\icmlauthor{Kry Yik Chau Lui}{borealis}
\icmlauthor{Pablo Hernandez-Leal}{borealis}
\end{icmlauthorlist}

\icmlaffiliation{to}{Department of Computation, University of Alberta, Canada}
\icmlaffiliation{borealis}{Borealis AI}

\icmlcorrespondingauthor{Yue Gao}{gao12@ualberta.ca}
\icmlcorrespondingauthor{Pablo Hernandez-Leal}{pablo.hernandez@borealisai.com}

\icmlkeywords{Reinforcement learning, risk, multiagent learning}

\vskip 0.3in
]



\printAffiliationsAndNotice{*Work performed as an intern at Borealis AI.}  

\begin{abstract}

Trading markets represent a real-world financial application to deploy reinforcement learning agents, however, they carry hard fundamental challenges such as high variance and costly exploration. Moreover, markets are inherently a multiagent domain composed of many actors taking actions and changing the environment. To tackle these type of scenarios agents need to exhibit certain characteristics such as \emph{risk-awareness}, \emph{robustness to perturbations} and \emph{low learning variance}. We take those as building blocks and propose a family of four algorithms. First, we contribute with two algorithms that use risk-averse objective functions and variance reduction techniques. Then, we augment the framework to multi-agent learning and assume an adversary which can take over and perturb the learning process. Our third and fourth algorithms perform well under this setting and balance theoretical guarantees with practical use. Additionally, we consider the multi-agent nature of the environment and our work is the first one extending empirical game theory analysis for multi-agent learning by considering risk-sensitive payoffs. 
\end{abstract}

\section{Introduction}

Reinforcement learning (RL) has moved from toy domains to real-world applications such as games~\cite{berner2019dota},  navigation~\cite{bellemare2020autonomous}, software engineering~\cite{bagherzadeh2020reinforcement}, industrial design~\cite{mirhoseini2020chip}, and finance~\cite{li2017deep}.  Each of these applications has inherent difficulties which are long-standing fundamental challenges in RL, such as: limited training time, costly exploration and safety considerations, among others. 

In particular, in finance, there are some examples of RL in stochastic control problems such as option pricing~\cite{li2009learning}, market making~\cite{spooner2018market}, and optimal execution~\cite{ning2018double}.
However, the most well-known finance application is algorithmic trading, where the goal is to design algorithms capable of automatically making trading decisions based on a set of mathematical rules computed by a machine~\cite{theate2021application}.

In algorithmic trading the environment represents the market (and the rest of the actors). The agent's task is to take actions related to how and how much to trade, and the objective is usually to maximize profit while considering risk. There are diverse challenges in this setting such as partial observability, a large action space, a hard definition of rewards and learning objectives~\cite{theate2021application}. In our work we focus on two sought properties for learning agents in realistic scenarios: risk assessment and robustness.

Risk assessment is a cornerstone in financial applications. A well-known approach is to consider risk while assessing the performance (profit)\footnote{Even when the usual financial term for profit is \emph{return}, this could be confused with the usual definition of return in RL (cumulative sum of discounted rewards).} of a trading strategy. Here, risk is a quantity related to the variance (or standard deviation) of the profit and it is commonly refereed to as ``volatility". In particular, the Sharpe ratio~\cite{sharpe1994sharpe} considers both the generated profit and the risk (variance) associated with a trading strategy. Note that this objective function (Sharpe ratio) is different from traditional RL where the goal is to optimize the expected return, usually, without considerations of risk. There are existing works that proposed risk-sensitive RL algorithms~\cite{mihatsch2002risk,di2012policy} and variance reduction techniques~\cite{anschel2017averaged}. In a similar spirit our proposed algorithms aim to reduce variance while also having convergence guarantees and improved robustness via adversarial learning.




Deep RL has been shown to be brittle in many scenarios~\cite{henderson2018deep}.
Therefore, improving robustness is essential for deploying agents in realistic scenarios. A line of work has improved robustness of RL agents via adversarial perturbations~\cite{morimoto2005robust,pinto2017robust}. In particular, the framework assumes an adversary (who is also learning) who is allowed to take over control at regular intervals. This approach has shown good experimental results in robotics~\cite{pan2019risk}, and our proposed algorithms extend on this idea while providing convergence guarantees.

Since our motivation is to use RL agents in trading markets (which can be seen as multi-agent interactions) we also evaluate these agents from the perspective of game theory. However, it may be too difficult to analyze in standard game theoretic framework since there is no normal form representation (commonly used to analyze games). Fortunately, empirical game theory~\cite{walsh2002analyzing,wellman2006methods} overcomes this limitation by using the information of several rounds of repeated interactions and assuming a higher level of strategies (agents' policies). These modifications have made possible the analysis of multi-agent interactions in complex scenarios such as markets~\cite{bloembergen2015trading}, and multi-agent games~\cite{tuyls2020bounds}. However, these works have not studied the interactions under risk metrics (such as Sharpe ratio) as we do in this work. 


In summary, we take inspiration from previous works to combine \emph{risk-awareness, variance reduction and robustness} techniques with four different algorithms.
Risk-Averse Averaged Q-Learning (\RAA) and Variance Reduced Risk-Averse Q-Learning (\RAAV) use risk-averse functions and variance reduction techniques. Then, we augment the framework to a multi-agent scenario where we assume an adversary that can perturb the learning process. We propose Risk-Averse Multi-Agent Q-Learning (\RAAM) which is a multi-agent version of adversarial learning with strong assumptions and theoretical guarantees. Risk-Averse Adversarial Averaged Q-Learning (\RAAA) relaxes those assumptions and proposes a more practical algorithm that keeps the multi-agent adversarial component to improve robustness.  Lastly, we present a theoretical result using empirical game theory analysis on games with risk-sensitive payoff.


\section{Preliminaries}

\subsection{Single-Agent Reinforcement Learning}

A Markov Decision Process is defined by a set of states $\mathcal{S}$ describing the possible configurations, a set of actions $\mathcal{A}$ and a set of observations $\mathcal{O}$ for each agent.
A stochastic policy $\pi_{\theta} : \mathcal{O}\times \mathcal{A} \rightarrow [0,1]$ parameterized by $\theta$ produces the next state according to the state transition function $\mathcal{T}: \mathcal{S}\times\mathcal{A}\rightarrow \mathcal{S}$. The agent obtains rewards as a function of the state and agent’s action $r:\mathcal{S}\times\mathcal{A}\rightarrow\mathbb{R}$, and receives a private observation correlated with the state $\mathbf{o} : \mathcal{S}\rightarrow\mathcal{O}$. The initial states are determined by a distribution $d_0 : \mathcal{S}\rightarrow [0,1]^{|\gS|}$. 
\subsection{Multi-Agent Reinforcement Learning}
In RL, each agent $i$ aims to maximize its own total expected return, e.g., for a Markov game with two agents, for a given initial state distribution $d_0$, the discounted returns are respectively :\begin{align}
     J^1(d_0, \pi^1, \pi^2)= \sum_{t=0}^{\infty}\gamma^t\expectation\left[r_t^1 | \pi^1, \pi^2, d_0\right]\\
    J^2(d_0, \pi^1,\pi^2)=\sum_{t=0}^{\infty}\gamma^t\expectation\left[r_t^2 | \pi^1, \pi^2, d_0\right]
\end{align}
where $\gamma$ is a discount factor, $r_t^1, r_t^2,\;t = 1,2,...$ are respectively immediate rewards for agent 1 \& 2. And a Nash equilibrium for Markov game (with two agents) is defined as following \begin{definition}\cite{MultiAgentQLearning98}
  A Nash equilibrium point of game $(J^1, J^2)$ is a pair of strategies $(\pi_*^1, \pi_*^2)$ such that for $\forall s\in\gS$, \begin{align}
      J^1(s, \pi^1_*, \pi^2_*)\geq J^1(s, \pi^1, \pi^2_*)\quad \forall \pi^1\\
      J^2(s, \pi^1_*, \pi^2_*)\geq J^2(s, \pi^1_*, \pi^2)\quad \forall \pi^2
  \end{align}
\end{definition}

\subsubsection{Multi-agent Extension of MDP}
A Markov game for $N$ agents is defined by a set of states $\mathcal{S}$ describing the possible configurations of all agents, a set of actions $\mathcal{A}_1, ..., \mathcal{A}_{N}$ and a set of observations $\mathcal{O}_1, ..., \mathcal{O}_N$ for each agent. To choose actions, each agent $i$ uses a stochastic policy $\pi_{\theta_i} : \mathcal{O}_i\times \mathcal{A}_i \rightarrow [0,1]$ parameterized by $\theta_i$, which produces the next state according to the state transition function $\mathcal{P}: \mathcal{S}\times\mathcal{A}_1\times...\times\mathcal{A}_N\rightarrow \mathcal{S}$. Each agent $i$ obtains rewards as a function of the state and agents' action $r_i:\mathcal{S}\times\mathcal{A}_1\times...\times\mathcal{A}_N\rightarrow\mathbb{R}$, and receives a private observation correlated with the state $\mathbf{o}_i : \mathcal{S}\rightarrow\mathcal{O}_i$. The initial states are determined by a distribution $d_0 : \mathcal{S}\rightarrow [0,1]^{|\gS|}$. In multi-agent Q learning, the Q tables are defined over joint actions for each of the agents. Each agent receives rewards according to its reward function, with transitions dependent on the actions chosen jointly by the set of agents. 

\subsection{Empirical Game Theory}
We analyze the multi-agent behaviours in a trading market using empirical game theory, where a \emph{player} corresponds to an agent, and a \emph{strategy} corresponds to a learning algorithm. Then, in a $p$-player game, players are involved in
a single round strategic interaction. Each player $i$ chooses a strategy $\pi^i$ from a set of $k$ strategy $S^i = \{\pi_1^i, ..., \pi_k^i \}$ and receives a stochastic payoff $R^i (\pi^1, ..., \pi^p ): S^1\times S^2\times...\times S^p\rightarrow \mathbb{R}$. The underlying game that is usually studied is $r^i (\pi^i, ..., \pi^p) = \mathbb{E}[R^i (\pi^1, ..., \pi^p)]$. In general, we denote the payoff of player $i$ as $\mu^i$ and $\mathbf{x}^{-i}$ as the joint strategy of all players except for player $i$.
\begin{definition}
A joint strategy $\mathbf{x} = (x^1, ..., x^p) = (x^i, \mathbf{x}^{-i})$ is a Nash equilibrium if for all $i$ :
\begin{align}
\mathbb{E}_{\bf{\pi}\sim\mathbf{x}}\left[\mu^i (\pi)\right] = \underset{\pi^i} \max~\mathbb{E}_{\pi^{-i}\sim\mathbf{x}^{-i}}\left[\mu^i (\pi^i, \mathbf{\pi}^{-i})\right]
\end{align}
\end{definition}

\begin{definition}
A joint strategy $\mathbf{x} = (x^1, ..., x^p) = (x^i, \mathbf{x}^{-i})$ is an $\epsilon$-Nash equilibrium if for all $i$:
\begin{align}\label{eq:NashEquilibrium}
 \underset{\pi^i} \max~\mathbb{E}_{\pi^{-i}\sim\mathbf{x}^{-i}}\left[\mu^i (\pi^i, \mathbf{\pi}^{-i})\right]-\mathbb{E}_{\bf{\pi}\sim\mathbf{x}}\left[\mu^i (\pi)\right]\le \epsilon
\end{align}
\end{definition}

Evolutionary dynamics have been used to analyze multi-agent interactions. A well-known model is replicator dynamics (RD)~\cite{weibull1997evolutionary} which describes how a population evolves through time under evolutionary pressure (in our analysis, a population is composed by learning algorithms). RD assumes that the reproductive success is determined by interactions and their outcomes. For example, the population of a certain type increases if they have a higher \emph{fitness} (in our case this means the expected return in certain interaction) than the population average; otherwise that population share will decrease.

To view the dominance of different strategies, it is common to plot the directional field of the payoff tables using the replicator dynamics for a number of strategy profiles $\mathbf{x}$ in the simplex strategy space~\cite{tuyls2020bounds}. In \cref{sec:risk_and_robustness_evaluation} we present results in this format evaluating our proposed algorithms.

\section{Related Work}

Our work is mainly situated in the broad area of safe RL~\cite{garcia2015comprehensive}. In particular, a subgroup of works aims to improve robustness of learned policies by assuming two opposing learning processes: one that aims to disturb the most and another one that tries to control the perturbations~\cite{morimoto2005robust}. This approach has been recently adapted to work with neural networks in the context of deep RL~\cite{pinto2017robust}. Moreover, Risk-Averse Robust Adversarial Reinforcement Learning (RARL)~\cite{pan2019risk} extended this idea by combining with Averaged DQN~\cite{anschel2017averaged}, an algorithm that proposes averaging the previous $k$ estimates to stabilize the training process. RARL trains two agents -- protagonist and adversary in parallel, and the goal for those two agents are respectively to maximize/minimize the expected return as well as minimize/maximize the variance of expected return. RARL showed good experimental results, but lacked theoretical guarantees and theoretical insights on the variance reduction and robustness. Multi-agent Q-learning \cite{MultiAgentQLearning98} is useful for finding the optimal strategy when there exists a unique Nash equilibrium in general sum stochastic games, and this approach could also be used in adversarial RL.

\citeauthor{wainwright2019variance}~(2019) proposed a variance reduction Q-learning algorithm (V-QL) which can be seen as a variant of the SVRG algorithm in stochastic optimization~\cite{NIPS2013_ac1dd209}. Given an algorithm that converges to $Q^*$, one of its iterates $\bar{Q}$ could be used as a proxy for $Q^*$, and then recenter the ordinary Q-learning updates by a quantity $-\hat{\gT}_k(\bar{Q}) + \gT(\bar{Q})$, where $\hat{\gT}_k$ is an empirical Bellman operator, $\gT$ is the population Bellman operator, which is not computable, but an unbiased approximation of it could be used instead. This algorithm is shown to be convergent and enjoys minimax optimality up to a logarithmic factor.

Lastly, another group of works proposed the use of risk-averse objective functions~\cite{mihatsch2002risk} with the Q-learning algorithm. Since these ideas are highly related to our proposed algorithms we will describe in greater detail in the next section. 

\subsection{Risk Averse Q Learning}
\label{sec:RAQL}

\citeauthor{shen2014risk}~(2014) proposed a Q learning algorithm that is shown to converge to the optimal of a risk-sensitive objective function, the training scheme is the same as Q learning, except that in each iteration, a utility function is applied to a TD-error (see \cref{alg:Risk_Averse_QLearning} in Appendix).

Since the goal is to optimize the expected return as well as minimizing the variance of the expected return, an expected utility of the return could be used as the objective function instead: 
\begin{align}
\label{eq:Risk_Averse_Objective}
   \tilde{J}_{\pi}= \frac{1}{\beta}\mathbb{E}_{\pi}\left[exp\left(\beta\sum_{t=0}^{\infty}\gamma^t r_t\right)\right].
\end{align}
By a straightforward Taylor expansion,  \cref{eq:Risk_Averse_Objective} yields
\begin{align*}
    \expectation[\sum_{t=0}^{\infty}\gamma^t r_t] + \frac{\beta}{2}\sV ar[\sum_{t=0}^{\infty}\gamma^t r_t] + O(\beta^2)
\end{align*}
where when $\beta<0$ the objective function is risk-averse, when $\beta=0$ the objective function is risk-neutral, and when $\beta>0$ the objective function is risk-seeking.

\citeauthor{shen2014risk}~(2014) proved that by applying a monotonically increasing concave utility function $u(x) = -exp(\beta x)$ where $\beta<0$ to the TD error, \cref{alg:Risk_Averse_QLearning} converges to the optimal point of \cref{eq:Risk_Averse_Objective}. Hence, it can be shown that:
\begin{theorem} (Theorem 3.2, \citeauthor{shen2014risk}~2014)
\label{thm:RARL_Converge}
Running \cref{alg:Risk_Averse_QLearning} from an initial Q table, $Q\rightarrow Q^*$ w.p. 1, where $Q^*$ is the unique solution to 
\begin{align*}
    \expectation_{s^{\prime}}\left[u\left(r(s, a) + \gamma\cdot\underset{a}{\max}Q^*(s^{\prime},a) - Q^*(s,a)\right)\right]-x_0 = 0
\end{align*}
$\forall (s,a)$. Where $s^{\prime}$ is sampled from $\gT[\cdot|s,a]$. And the corresponding policy $\pi^*$ of $Q^*$ satisfies $\tilde{J}_{\pi^*}\geq \tilde{J}_{\pi}\;\forall \pi$.
\end{theorem}

\subsection{Multi-Agent Q-Learning}
\citeauthor{MultiAgentQLearning98}~(1998) proposed Nash-Q, a Multi-Agent Q-learning algorithm (\cref{alg:MultiAgent_QLearning} in Appendix) in the framework of general-sum stochastic games. When there exists a unique Nash equilibrium in the game, this algorithm is useful for finding the optimal strategy. Nash-Q assumes an agent can observe the other agent's immediate rewards and previous actions during learning. Each learning agent maintains two Q-tables, one for its own Q values, and one for the other agents'. \citeauthor{MultiAgentQLearning98}~(1998) showed that under strong assumptions (\cref{ass:Bimatrix_Nash_Assumption} in Appendix), Nash-Q converges to the Nash Equilibrium. We leave the full version of the algorithm and the convergence theorem to \cref{sec:Multi_Agent_QLearning}.

\begin{table}[]
    \caption{Comparison of related algorithms. Our proposed algorithms are marked with \textbf{bold} and are described in \cref{sec:algorithms}.}
    \label{tab:comparison_algs}
\scriptsize
    \centering
    \begin{tabular}{@{}p{2cm}|p{3.0cm}|p{2.5cm}@{}} \toprule
       \bf Algorithm  & \bf Description & \bf Guarantees \\ \hline
        Risk averse Q-Learning~\cite{shen2014risk}& Q-Learning with a utility function applied to TD Error in Q update & Convergence to optimal of a risk-averse objective function\\
        Variance reduced Q-learning~\cite{wainwright2019variance} & Use average estimation of multiple $Q$ tables in Q-table updates to reduce variance & Convergent to optimal of expected return. Convergence rate is minimax optimal up to a logarithmic factor.\\
        Nash Q-learning~\cite{MultiAgentQLearning98} & Two-agent Q-Learning in multi-agent MDP setting & Convergence to Nash equilibrium of the two-agent game (if exists) \\
        Risk-Averse Robust Adversarial Reinforcement Learning (RARL)~\cite{pan2019risk} & Q-Learning with risk-averse/risk-seeking behaviors of protagonist/adversary with multiple $Q$ tables & No convergence guarantee 
        \\
        \bf Risk-Averse Averaged Q-Learning (\RAA) & Q-Learning with a utility function + a more stable choice of actions with multiple Q tables & Convergence to optimal of a risk-averse objective function and reduced training variance.\\
        \bf Variance Reduced Risk-Averse Q-Learning (\RAAV) & Use average estimation of multiple $Q$ tables in Q updates; Apply utility function in Q updates &  No convergence guarantee \\
        \bf Risk-Averse Multiagent Q-learning (\RAAM)  & Multi-agent Nash Q-Learning with a utility function + a risk-averse/risk-seeking behaviors of protagonist/adversary + multiple Q tables & Convergence to Nash equilibrium (if exists) of the two-agent game (with Risk-Averse/Seeking payoffs respectively)\\ 
        \bf Risk-Averse Adversarial Averaged Q-Learning (\RAAA)  & Multi-agent Q-Learning with a utility function + a risk-averse/risk-seeking behaviors of protagonist/adversary + multiple Q tables & No convergence guarantee\\ 
        \bottomrule
    \end{tabular}
\end{table}

\section{Proposed Algorithms}
\label{sec:algorithms}

Here we describe our proposed algorithms continuing the results discussed in the previous sections. We first present two algorithms \RAA and \RAAV which use a risk-averse utility functions and reduce variance by training multiple Q tables in parallel. Then, we present \RAAM which is a multi-agent algorithm that assumes an adversary which can perturb the learning process.
While \RAAM is proven to have convergence guarantees, it also needs strong assumptions that might not hold in reality.
Therefore, our last proposal, \RAAA, keeps the adversarial component to improve robustness while relaxing the strong assumptions.
As a summary, \cref{tab:comparison_algs} presents closely related works and the comparison with our proposed algorithms.

\subsection{Risk-Averse Averaged Q-Learning (\RAA) }\label{sec:RA2-Q}

\begin{algorithm*}[ht]
\scriptsize
\caption{Risk-Averse Averaged Q-Learning (RA2-Q)}
\label{alg:Risk_Averse_Averaged_QLearning}
\textbf{Input :} Training steps $T$; Exploration rate $\epsilon$; Number of models $k$; risk control parameter $\lambda_P$; Utility function parameter $\beta$.\begin{spacing}{0.8}
\begin{algorithmic}[1]
\STATE Initialize $Q^i= \mathbf{0}$, $N^i= \mathbf{0}$, $\alpha^i = \mathbf{1}$ for $\forall i = 1,..., k$. 
\STATE Initialize Replay Buffer $RB= \emptyset$; Randomly sample action choosing head integers $H\in[1,k]$
\FOR{$t=1$ to $T$}
\STATE $Q = Q^{H}$
\STATE Compute $\hat{Q}$ by 
\begin{align}
    \hat{Q}(s,a) = Q(s,a) - \lambda_P\cdot \frac{\sum_{i=1}^{k}(Q^i(s,a) - \bar{Q}(s,a))^2}{k-1}
\end{align} where $\lambda_P>0$ is a constant; $\bar{Q}(s,a) = \frac{1}{k}\sum_{i=1}^{k}Q^i(s,a)$
\STATE Select action $a_t$ according to $\hat{Q}$ by applying $\epsilon$-greedy strategy.
\STATE Execute actions and get $(s_t, a_t, r_t, s_{t+1})$, append to the replay buffer $RB = RB\cup \{(s_t, a_t, r_t, s_{t+1})\}$
\STATE Generate mask $M\in \sR^k \sim Poisson(1)$
\FOR{$i=1,...,k$}
\IF{$M_i = 1$}
\STATE Update $Q^i$ by 
\begin{align}
\label{eq:RAA_Q_Update}
    Q^i(s_t,a_t) = Q^i(s_t,a_t) + \alpha^i(s_t,a_t)\cdot\left[u\left(r(s_t,a_t) + \gamma\cdot\underset{a}{\max} Q^i(s_{t+1}, a) - Q^i(s_t, a_t)\right)-x_0\right]
\end{align}where $u$ is a utility function, here we use $u(x) = -e^{\beta x}$ where $\beta<0$; $x_0 = -1$
\STATE $N^i(s_t,a_t) = N^i(s_t,a_t) + 1$; Update learning rate $\alpha^i(s_t,a_t) = \frac{1}{N^i(s_t,a_t)}$.
\ENDIF
\ENDFOR
\STATE Update $H$ by randomly sampling integers from 1 to $k$. 
\ENDFOR
\STATE \textbf{Return} $\frac{1}{k}\sum_{i=1}^{k}Q^i$
\end{algorithmic}
\end{spacing}
\end{algorithm*}

Although in RAQL (\cref{alg:Risk_Averse_QLearning}) we discussed the convergence to the optimal of risk-sensitive objective function with probability 1, the proof assumes visiting every state infinitely many times whereas the actual training time is finite. 
Our main idea is that we can reduce the training variance further by choosing more risk-averse actions during the finite training process.

Averaged DQN~\cite{anschel2017averaged} reduces training variance by averaging multiple Q tables in the update.
In a similar spirit, our proposed \RAA also trains multiple Q tables in parallel. However, we do not directly use the same update rule since that would break the convergence guarantee, in contrast, we train $k$ Q tables in parallel using \cref{eq:RAA_Q_Update} as update rule. To select more \emph{stable} actions we use the sample variance of $k$ Q tables as an approximation to the true variance and then compute a risk-averse $\hat{Q}$ table and select actions according to it. A detailed description is presented in \cref{alg:Risk_Averse_Averaged_QLearning}.

The objective function here is also \cref{eq:Risk_Averse_Objective}, and it can be shown that \cref{alg:Risk_Averse_Averaged_QLearning} also converges to the optimal.

\begin{theorem}\label{thm:Convergence_of_RAA}
Running \cref{alg:Risk_Averse_Averaged_QLearning} for an initial Q table, then for all $i\in\{1,...,k\}$, $Q^i\rightarrow Q^*$ w.p. 1, hence the returned table $\frac{1}{k}\sum_{i=1}^{k}Q^i\rightarrow Q^*$ w.p. 1, where $Q^*$ is the unique solution to \begin{align*}
    &\expectation_{s^{\prime}}\left[u\left(r(s, a) + \gamma\cdot\underset{a}{\max}Q^*(s^{\prime},a) - Q^*(s,a)\right)\right]-x_0 = 0
\end{align*} for all $(s,a)$. Where $s^{\prime}$ is sampled from $\gT[\cdot|s,a]$. And the corresponding policy $\pi^*$ of $Q^*$ satisfies $\tilde{J}_{\pi^*}\geq \tilde{J}_{\pi}\;\forall \pi$.
\end{theorem}

\cref{thm:Convergence_of_RAA} follows directly from \cref{thm:RARL_Converge} (see \cref{sec:proof_of_convergence_RAA} for detail).

\subsection{Variance Reduced Risk-Averse Q-Learning (\RAAV)}\label{sec:RAAV}

\begin{algorithm*}[t]
\scriptsize
\caption{Variance Reduced Risk-Averse Q-Learning (\RAAV)}
\label{alg:Variance_Reduced_RAQL}
\textbf{Input :} Training epochs $T$; Exploration rate $\epsilon$; Number of models $k$; Epoch length $K$; Recentering sample size $N$; Utility function parameter $\beta<0$;  
\begin{spacing}{0.8}
\begin{algorithmic}[1]
\STATE Initialize $\bar{Q}_0 = \mathbf{0}$; $m = 1$; $RB = \emptyset$.
\FOR{$m = 1$ to $T$}
\STATE Select action according to $\bar{Q}_{m-1}$ by applying $\epsilon-$greedy strategy
\STATE Execute action and get $(s,a,r(s,a),s^{\prime})$ and update the replay buffer $RB = RB\cup (s,a,r(s,a),s^{\prime})$.
\FOR{$i = 1,..., N$}
\STATE Define the empirical Bellman operator $\ddot{\mathcal{T}_i}$ as 
$$\ddot{\mathcal{T}_i}(Q)(s,a)=u\left(r(s,a)+\gamma\cdot\underset{a^{\prime}}{\max}\;Q(s_i,a^{\prime})\right) - x_0$$
\qquad where $s_i$ is randomly sampled from $\gT[\cdot|s,a]$; $u$ is the utility function, and $u(x) = -e^{\beta x}$, $\beta<0$ and $x_0 = -1$
\ENDFOR
\STATE Define $\tilde{\mathcal{T}}_{N}(\bar{Q}_{m-1})=\frac{1}{N}\sum_{i\in\mathcal{D}_{N}}\ddot{\mathcal{T}_i}(\bar{Q}_{m-1})$, where $\mathcal{D}_{N}$ is a collection of $N$ i.i.d. samples (i.e., matrices with samples for each state-action pair $(s,a)$ from $RB$).
\STATE Define $Q_1 = \bar{Q}_{m-1}$.
\FOR{$k = 1,..., K$}
\STATE Compute stepsize $\lambda_k = \frac{1}{1+(1-\gamma)k}$
\STATE  \begin{align}\label{eq:Update_Rule}
    Q_{k+1} = (1-\lambda_k)\cdot Q_{k} + \lambda_k\cdot\left[\ddot{\mathcal{T}_k}(Q_k)-\ddot{\mathcal{T}_k}(\bar{Q}_{m-1}) + \tilde{\mathcal{T}}_{N}(\bar{Q}_{m-1})\right].
\end{align}
where $\ddot{\mathcal{T}_k}$ is empirical Bellman operator constructed using a sample not in $\mathcal{D}_N$, thus the random operators $\ddot{\mathcal{T}_k}$ and $\tilde{\mathcal{T}}_N$ are independent
\ENDFOR
\STATE $\bar{Q}_{m} = Q_{K+1}$; $m$ = $m +1$
\ENDFOR
\STATE \textbf{Return} $\bar{Q}_{m}$
\end{algorithmic}
\end{spacing}
\end{algorithm*}

 \citeauthor{wainwright2019variance}~(2019) proposed Variance Reduced Q-learning which trains multiple Q tables in parallel and uses the averaged Q table in the update rule.
 It is shown that it guarantees a convergence rate which is minimax optimal.
 Inspired by that work, we propose our \RAAV (\cref{alg:Variance_Reduced_RAQL}) which applies a utility function to the TD error during Q updates for the purpose of further reducing variance. To select more \emph{stable} actions during training, we use the sample variance of $k$ Q tables as an approximation to the true variance and then compute a risk-averse $\hat{Q}$ table and select actions according to it. We'll discuss more details in \cref{sec:Discussion}.


\subsection{Multi-Agent Risk-Averse Q-Learning (\RAAM)}

\begin{algorithm*}[ht]
\scriptsize
\caption{Risk-Averse Multi-Agent Q-Learning (RAM-Q)}
\label{alg:MultiAgent_QLearning_RiskAverse}
\textbf{Input :} Training steps $T$; Exploration rate $\epsilon$; Number of models $k$; Utility function parameters $\beta^P<0; \beta^A > 0$.
\begin{spacing}{0.8}
\begin{algorithmic}[1]
\STATE For $\forall (s,a_P,a_A)$, initialize $Q^P(s,a_P,a_A) = 0$; $Q^A(s,a_P,a_A) = 0$; $N(s,a_P,a_A) = 0$.
\FOR{$t = 1$ to $T$}
\STATE At state $s_t$, compute $\pi^P(s_t)$, $\pi^A(s_t)$, which is a mixed strategy Nash equilibrium solution of the bimatrix game $(Q^P(s_t), Q^A(s_t))$.
\STATE Choose action $a_t^P$ based on $\pi^P(s_t)$ according to $\epsilon$-greedy and choose action $a_t^A$ based on $\pi^A(s_t)$ according to $\epsilon$-greedy
\STATE Observe $r_t^P, r_t^A$ and $s_{t+1}$.
\STATE At state $s_{t+1}$, compute $\pi^P(s_{t+1})$,$\pi^A(s_{t+1})$, which are mixed strategies Nash equilibrium solutions of the bimatrix game $(Q^P(s_{t+1}), Q^A(s_{t+1}))$.
\STATE $N(s_t,a^P_t, a^A_t)= N(s_t,a^P_t, a^A_t) + 1$
\STATE Set learning rate $\alpha_t = \frac{1}{N(s_t,a^P_t, a^A_t)}$.
\STATE Update $Q^P, Q^A$ such that \begin{align}\label{eq:UpdateRule_MultiAgent_RiskAverse_1}
    Q^P(s_t,a^P_t, a^A_t) = Q^P(s_t,a^P_t, a^A_t) + \alpha_t\cdot \left[u^P\left(r_t^P + \gamma\cdot\pi^P(s_{t+1})Q^P(s_{t+1})\pi^A(s_{t+1})-Q^P(s_t,a^P_t, a^A_t)\right) - x_0\right]
\end{align}
where $u^P$ is a utility function, here we use $u^P(x) = -e^{\beta^P x}$ where $\beta^P<0$; $x_0 = -1$.
\begin{align}\label{eq:UpdateRule_MultiAgent_RiskAverse_2}
    Q^A(s_t,a^P_t, a^A_t) = Q^A(s_t,a^P_t, a^A_t) + \alpha_t\cdot\left[ u^A\left(r_t^A + \gamma\cdot\pi^P(s_{t+1})Q^A(s_{t+1})\pi^A(s_{t+1})-Q^A(s_t,a^P_t, a^A_t)\right) - x_1\right]
\end{align}
where $u^A$ is a utility function, here we use $u^A(x) = e^{\beta^A x}$ where $\beta^A>0$; $x_1 = 1$.
\ENDFOR
\STATE \textbf{Return} $(Q^P, Q^A)$
\end{algorithmic}
\end{spacing}
\end{algorithm*}

In complex scenarios such as financial markets learned RL policies can be brittle. To improve robustness, we adapt ideas from adversarial learning to a multi-agent learning problem similar to \cite{MultiAgentQLearning98}. 

In the adversarial setting we assume there are two learning processes happening simultaneously, a main protagonist \emph{(P)} and an adversary \emph{(A)}: the goal of protagonist is to maximize the total return as well as minimize the variance; the goal of adversary is to minimize the total return of protagonist as well as maximizing the variance. Here, we assume that each agent can observe its opposite's immediate reward.

Let $r_t^P$ be the immediate reward received by protagonist at step $t$, and let $r_t^A$ be the immediate reward received by adversary at step $t$. Then we choose the objective functions as follows:

The objective function for the protagonist is,
\begin{align}\label{eq:RAM-Q_Objective_protagonist}
    \tilde{J}_{\pi}^{P} = \frac{1}{\beta^P}\expectation_{\pi}\left[exp\left(\beta^P\sum_{t=0}^{\infty}\gamma^t\cdot r_t^P \right)\right] \qquad \beta^P < 0
\end{align}
by a Taylor expansion, \cref{eq:RAM-Q_Objective_protagonist} yields,
\begin{align*}
   \tilde{J}_{\pi}^{P} &= \expectation\left[\sum_{t=0}\gamma^t\cdot r_t^P\right] + \frac{\beta^P}{2}\sV ar\left[\sum_{t=0}\gamma^t\cdot r_t^P\right]
    + O((\beta^P)^2). 
\end{align*}
Similarly, the objective function for the adversary is,
\begin{align}\label{eq:RAM-Q_Objective_adversary}
    \tilde{J}_{\pi}^{A} = \frac{1}{\beta^{A}}\expectation_{\pi}\left[exp\left(\beta^{A}\sum_{t=0}^{\infty}\gamma^t  r_t^A\right)\right] \qquad \beta^A >0
\end{align}
and by Taylor expansion, \cref{eq:RAM-Q_Objective_adversary} yields, 
\begin{align*}
    \tilde{J}_{\pi}^{A} =& \expectation\left[\sum_{t=0}\gamma^t\cdot  r_t^A\right] +\frac{\beta^{A}}{2}\sV ar\left[\sum_{t=0}\gamma^t\cdot r_t^A\right] + O((\beta^A)^2) .
\end{align*}
Using the same spirit in \cite{MultiAgentQLearning98}, we proposed \cref{alg:MultiAgent_QLearning_RiskAverse} and the following guarantee holds :
\begin{theorem}\label{thm:RAMconvergenceRate}
If the two-agent game $(\tilde{J}^P, \tilde{J}^A)$ has a Nash equilibrium solution, then running \cref{alg:MultiAgent_QLearning_RiskAverse} from initial Q tables $Q^P, Q^A$ will converge to $Q_P^*$ and $Q_A^*$ w.p. 1. s.t. the Nash equilibrium solution $(\pi^{P}_*, \pi^{A}_*)$ for the bimatrix game $(Q_P^*, Q_A^*)$ is the Nash equilibrium solution to the game $(\tilde{J}^P_{\pi}, \tilde{J}^A_{\pi})$, and the equilibrium payoff are $\tilde{J}^P(s,\pi^{P}_*, \pi^{A}_*)$, $\tilde{J}^A(s,\pi^{P}_*, \pi^{A}_*)$.
\end{theorem}

Although \cref{thm:RAMconvergenceRate} gives a solid convergence guarantee, it suffers from drawbacks like expensive computational cost and idealized assumptions, e.g., in trading markets, there may not exist a Nash equilibrium to $(\tilde{J}^P, \tilde{J}^A)$, and during the training process, assumptions about the Nash equilibrium (\cref{ass:Bimatrix_Nash_Assumption} in Appendix) break easily~\cite{bowling2000convergence}. Hence, we design another novel algorithm \RAAA which relaxes these assumptions (at the expense of loosing theoretical guarantees) while enhancing robustness and performing well in reality.

\subsection{Risk-Averse Adversarial Averaged Q-Learning (\RAAA)}
\label{sec:RA3-Q}
\begin{algorithm}[t]
\scriptsize
\caption{Risk-Averse Adversarial Averaged Q-Learning (\RAAA) \textit{Short Version}}
\label{alg:Risk_Averse_Adversarial_Averaged_QLearning}
\textbf{Input :} Training steps $T$; Exploration rate $\epsilon$; Number of models $k$; Risk control parameters $\lambda_P, \lambda_A$; Utility function parameters $\beta^P < 0; \beta^A > 0$. 
\begin{spacing}{0.8}
\begin{algorithmic}[1]
\STATE Initialize $Q_P^i, Q_A^i\;\forall i = 1,...,k$; $N = \mathbf{0}\in\sR^{|\gS|\times|\gA|\times|\gA|}$. Randomly sample action choosing head integers $H_P,H_A\in\{1,...,k\}$.
\FOR{$t=1$ to $T$}
\STATE Set $Q_P = Q_P^{H_P}$. Then compute $\hat{Q}_P$, the risk-averse protagonist $Q$ table by the $k$ Q tables $Q_P^i, i = 1,...,k$.
\STATE Set $Q_A = Q_A^{H_A}$. Then compute $\hat{Q}_A$, the risk-seeking protagonist $Q$ table by the $k$ Q tables $Q_A^i, i = 1,...,k$.
\STATE Select actions $a_P,a_A$ according to $\hat{Q}_P,\hat{Q}_A$ by applying $\epsilon$-greedy strategy.
\STATE Generate mask $M\in\sR^k \sim Poisson(1)$ and update $Q_P^i, Q_A^i, i = 1,..., k$ according to mask $M$ using update rules \cref{eq:RA3QProtagonist_UpdateRule} and \cref{eq:RA3QAdversary_UpdateRule}. 
\STATE Update $H_P$ and $H_A$
\ENDFOR
\STATE \textbf{Return} $\frac{1}{k}\sum_{i=1}^{k}Q_P^i$; $\frac{1}{k}\sum_{i=1}^{k}Q_A^i$.
\end{algorithmic}
\end{spacing}
\end{algorithm}

We start from the same objective functions for the protagonist, \cref{eq:RAM-Q_Objective_protagonist}, and adversary, \cref{eq:RAM-Q_Objective_adversary}. In order to optimize $\tilde{J}^{P}$ and $\tilde{J}^{A}$, we apply utility functions to TD errors when updating Q tables, and combining the idea of training multiple Q tables in parallel as \cref{alg:Risk_Averse_Averaged_QLearning} to select actions with low variance, we get a novel \cref{alg:Risk_Averse_Adversarial_Averaged_QLearning} (full version \cref{alg:Risk_Averse_Adversarial_Averaged_QLearning_fullversion} in \cref{sec:Discussion_of_RAAA}).

Note that \RAAA combines (i) risk-averse using utility functions (ii) variance reduction by training multiple Q tables and (iii) robustness by adversarial learning. Intuitively, as the adversary is getting stronger, the protagonist experiences harder challenges, thus enhancing robustness. Compared to \cref{alg:MultiAgent_QLearning_RiskAverse}, where the returned policy $(\pi^P, \pi^A)$ is a Nash equilibrium of the $(\tilde{J}^P, \tilde{J}^A)$, \cref{alg:Risk_Averse_Adversarial_Averaged_QLearning} does not have a convergence guarantee, however, it has several practical advantages including computational efficiency, simplicity (no strong assumptions) and more stable actions during training. For a longer discussion see \cref{sec:Discussion} and \cref{sec:Discussion_of_RAAA}.

\section{Performance Evaluated by Empirical Game Theory}

When the environment is populated by many learning agents, how do we evaluate their performance and decide which strategy is the best? Although different approaches can be used, we focused on empirical game theory (EGT) to address this question. 

In EGT each agent is a player involved in rounds of strategic interaction (games). By meta-game analysis, we can evaluate the superiority of each strategy. Our contribution is to theoretically prove that the Nash-Equilibrium of risk averse meta-game is an approximation of the Nash-Equilibrium of the population game, to our knowledge, this is the first work doing this type of risk-averse analysis.

\subsection{Replicator dynamics}
In EGT, we can visualize the dominance of strategies by plotting the meta-game payoff tables together with the replicator dynamics. A meta game payoff table could be seen as a combination of two matrices $(N|R)$, where each row $N_i$ contains a discrete distribution of $p$ players over $k$ strategies, and each row yields a discrete profile $(n_{\pi_1}, ..., n_{\pi_k})$ indicating exactly how many players play each strategy with $\sum_{j}n_{\pi_j} = p$. A strategy profile $\mathbf{u} = \left(\frac{n_{\pi_1}}{p}, ..., \frac{n_{\pi_k}}{p}\right)$. And each row $R_i$ captures the rewards corresponding to the rows in $N$. 

For example, for a game $A$ with 2 players, and 3 strategies $\{\pi_1, \pi_2, \pi_3\}$ to choose from, the meta game payoff table could be constructed as follows : In the left side of the table, we list all of the possible combinations of strategies. If there are $p$ players and $k$ strategies, then there are $\binom{p+k-1}{p}$ rows, hence in game $A$, there are 6 rows. See \cref{sec:meta_game_examples} for a concrete example.

 Once we have a meta-game payoff table and the replicator dynamics, a directional field plot is computed where arrows in the strategy space indicates the direction of flow, or change, of the population composition over the strategies (see \cref{sec:meta_game_examples} for two examples of directional field plots in multi-agent problems). In \cref{sec:risk_and_robustness_evaluation} we present trading market experiments and results based on meta-game analysis with the performance of RAQL, \RAA and \RAAV.  

\subsection{Nash Equilibrium with risk neutral payoff}

Previously, \citeauthor{tuyls2020bounds}~(2020) showed that for a game $r^i (\pi^i, ..., \pi^p) = \expectation [R^i(\pi^1, ..., \pi^p)]$, with a meta-payoff (empirical payoff) $\hat{r}^i (\pi^i, ..., \pi^p)$, the Nash Equilibrium of $\hat{r}$ is an approximation of Nash Equilibrium of $r$.

\begin{lemma}\cite{tuyls2020bounds}
\label{lem:approx_equili_normalgame}
If $\mathbf{x}$ is a Nash Equilibrium for the game $\hat{r}^i (\pi^1, ..., \pi^p)$, then it is a $2\epsilon$-Nash equilibrium for the game $r^i (\pi^1, ..., \pi^p)$, where $\epsilon =  \underset{\pi,i}{sup}~|\hat{r}^{i}(\pi) - r^i (\pi)|$.
\end{lemma}

\cref{lem:approx_equili_normalgame} implies that if for each player, we can bound the estimation error of empirical payoff, then we can use the Nash Equilibrium of meta game as an approximation of Nash Equilibrium of the game.

\subsection{Risk averse payoff EGT}

Recall our objective is to consider risk averse payoff to evaluate strategies. Hence, instead of letting $$r^i (\pi^i, ..., \pi^p) = \mathbb{E}[R^i (\pi^1, ..., \pi^p)],$$ we choose $$h^i(\pi^i, ..., \pi^p) = \mathbb{E}[R^i (\pi^1, ..., \pi^p)] - \beta\cdot\mathbb{V}ar[R^i (\pi^1, ..., \pi^p)]$$ (where $\beta>0$) as the game payoff. Moreover, we use\begin{align}\label{eq:Risk_Averse_Payoff}
    \hat{h}^i(\pi^i, ..., \pi^p) = \bar{R^i}  - \beta\cdot \left[\frac{1}{n-1}\sum_{j=1}^{n}\left(R^i_j - \bar{R^i}\right)^2\right]
\end{align}   as meta-game payoff, where $\bar{R^i} = \frac{1}{n}\sum_{j=1}^{n}R^i_j$ and $R_j^i$ is the stochastic payoff of player $i$ in $j-$th experiment. To our knowledge, there is no previous work on empirical game theory analysis with risk sensitive payoff. Below we give the first theoretical analysis showing that for our risk-averse payoff game, we can still approximate the Nash Equilibrium by meta game.

\begin{theorem}
\label{thm:approx_equili_riskaversegame}
Under \cref{ass:stochastic_reward_bounded}, for a Normal Form Game with $p$ players, and each player $i$ chooses a strategy $\pi^i$ from a set of strategies $S^i = \{\pi^i_1, ..., \pi^i_k\}$ and receives a meta payoff $h^i(\pi^1, ..., \pi^p)$ (\cref{eq:Risk_Averse_Payoff}). If $\mathbf{x}$ is a Nash Equilibrium for the game $\hat{h}^i (\pi^1, ..., \pi^p)$, then it is a $2\epsilon$-Nash equilibrium for the game $h^i (\pi^1, ..., \pi^p)$ with probability $1-\delta$ if we play the game for  $n$ times, where 
\begin{align}
\begin{split}
     n  \ge \max
     \left\{ -\frac{8R^2}{\epsilon^2}\log\left[\frac{1}{4}\left(1-(1-\delta)^{\frac{1}{|S^1|\times ...\times |S^p|\times p}}\right)\right], \right.
     \\ 
     \left. \frac{64\beta^2\omega^2\cdot\Gamma(2)}{\epsilon^2\left[1-(1-\delta)^{\frac{1}{|S^1|\times...\times |S^p|\times p}}\right]}\right\}
\end{split}
\end{align}
\end{theorem}

\section{Experiments}
\subsection{Setup}

Our experiments use the open-sourced ABIDES~\cite{byrd2019abides} market simulator in a simplified setting. The environment is generated by replaying publicly available real trading data for a single stock ticker.\footnote{https://lobsterdata.com/info/DataSamples.php} The setting is composed of one non-learning agent that replays the market deterministically~\cite{balch2019evaluate} and learning agents. The learning agents considered are: RAQL, \RAA, \RAAV, and \RAAA. 

We follow a similar setting to existing implementations in ABIDES\footnote{https://github.com/abides-sim/abides/blob/\\master/agent/examples/QLearningAgent.py} where the state space is defined by two features: current holdings and volume imbalance. Agents take one action at every time step (every second) selecting among: \emph{buy/sell} with limit price $ base + i\cdot K$, where $i \in\{1,2,..,6\}$ or \emph{do nothing}. The immediate reward is defined by the change in the value of our portfolio (mark-to-market) and comparing against the previous time step. Our comparisons are in terms of Sharpe ratio,\footnote{We could have compared in terms of the objectives functions (e.g., \cref{eq:RAA_Q_Update}) but instead we used Sharpe ratio which is more common in practice.} which is a widely used measure in trading markets. 

\subsection{Risk and robustness evaluation}
\label{sec:risk_and_robustness_evaluation}

\begin{table}[]
\scriptsize
    \caption{Meta-payoff of 2 players, 3 strategies, respectively RAQL~\cite{shen2014risk}, \RAA and \RAAV over 80 simulations. The return here used is Sharpe Ratio.}
    \label{table:MetaPayoff_oneAgentAlgos}
    \centering
    \begin{tabular}{c c c|c c c}
        \toprule
       $N_{i1}$  & $N_{i2}$ & $N_{i3}$ & $R_{i1}$ & $R_{i2}$ & $R_{i3}$\\ \hline
       2 & 0 & 0 &0.9130 & 0 & 0\\
       1  &  1 & 0 & 0.7311 & 0.7970 & 0\\
       0 & 2 & 0 & 0 &1.0298 & 0  \\
       1 & 0 & 1 & 0.6791 & 0 & 1.0786\\
       0 & 0 & 2 & 0 & 0 & 2.2177 \\
       0 & 1 & 1 & 0 & 0.7766  & 1.4386\\
       \bottomrule
    \end{tabular}
\end{table}
\begin{figure}
    \centering
    \subfigure[]{
    \includegraphics[width = 3.85cm]{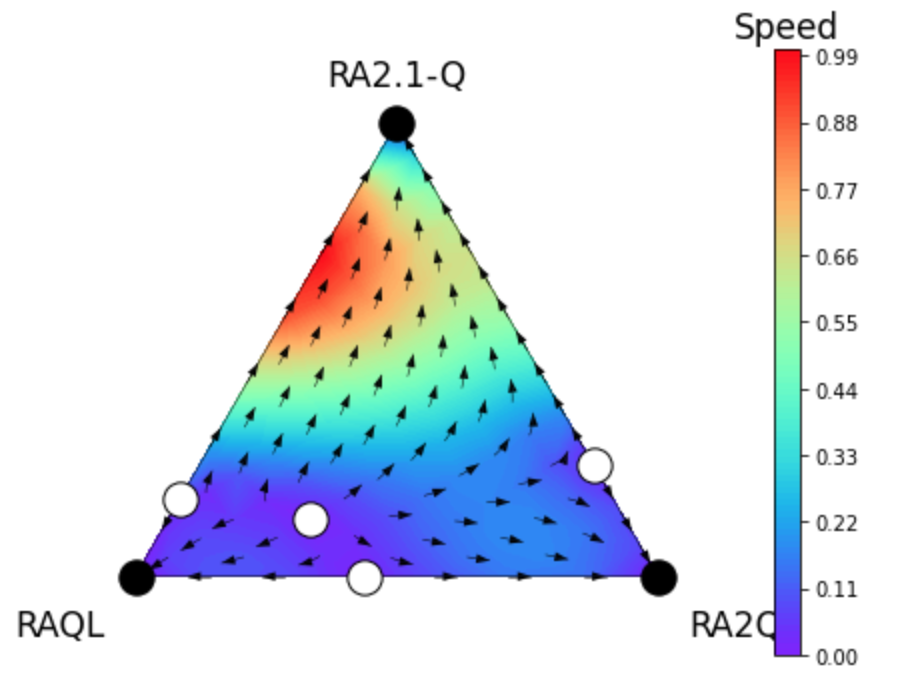}
     }
    \subfigure[]{
    \includegraphics[width = 3.85cm]{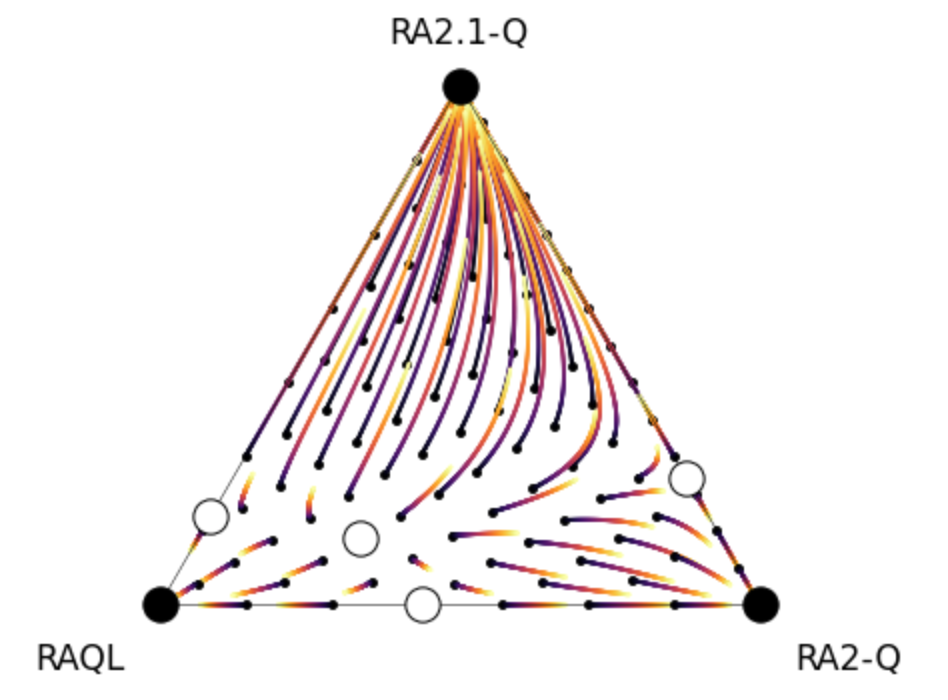}
    }
      \caption{(a) Directional field plot and (b) Trajectory plot of the simplex of 3 strategies based on the meta-game payoff from \cref{table:MetaPayoff_oneAgentAlgos}. It can be seen that \RAAV (top) is the the strongest attractor. White circles represent equilibria.}
    \label{fig:DirectionalFieldRiskAverseAlgo}
 
\end{figure}

\cref{table:MetaPayoff_oneAgentAlgos} shows the meta-payoff table of a two player-game among three strategies: RAQL, \RAA and \RAAV. The results show that our two proposed algorithms \RAA and \RAAV obtained better results than RAQL. With those payoffs we obtained the directional and trajectory plots shown in \cref{fig:DirectionalFieldRiskAverseAlgo}, where black solid circles denote globally-stable equilibria, and the white circles denote unstable equilibria (saddle-points), in (a) the plot is colored according to the speed at which the strategy mix is changing at each point; in (b) the lines show trajectories for some points over the simplex.

\begin{table}[t]
\scriptsize
    \caption{Comparison in terms for Sharpe ratio with two types of perturbations: The trained adversary from \RAAA is used in testing time. Zero-intelligence agents are added to the simulation to perturb the market. \RAAA obtains better results in both cases due to its enhanced robustness.}\label{table:RA2Q_RA3Q}
    \centering
    \begin{tabular}{c | c | c}
        \toprule
       Algorithm/Setting & Adversarial Perturbation & ZI Agents Perturbation\\
        \hline
      \RAA   & 0.5269 & 0.9538\\
      \RAAA & 0.9347 & 1.0692\\
      
       \bottomrule
    \end{tabular}
\end{table}

Our last experiment compares \RAA and \RAAA in terms of robustness. In this setting we trained both agents under the same conditions as a first step. Then in testing phase we added two types of perturbations, one adversarial agent (trained within \RAAA) or adding noise (aka. zero-intelligence) agents in the environment. In both cases, the agents will act in a perturbed environment. The results are presented in \cref{table:RA2Q_RA3Q} in terms of Sharpe ratio using cross validation with 80 experiments.

\section{Discussion}
\label{sec:Discussion}

Here we briefly discuss some trade-offs between practical and theoretical results about our proposed algorithms.

As mentioned in \cref{sec:RAAV}, we did not show that \cref{alg:Variance_Reduced_RAQL} (\RAAV) has a convergence guarantee, however, it obtained good empirical results (better than RAQL and \RAA). It is an open question whether \RAAV converges to the optimal of \cref{eq:Risk_Averse_Objective}, furthermore, it could be interesting to study whether it also enjoys minimax optimality convergence rate up to a logarithmic factor as in \cite{wainwright2019variance}. Similarly, \RAAA does not have a convergence guarantee in the multi-agent learning scenario (when protagonist and adversary are learning simultaneously). However, \RAAA obtained better empirical results than \RAA highlighting its robustness. In \cref{sec:Discussion_of_RAAA} we show a related result showing that \cref{eq:RA3QProtagonist_UpdateRule} or \cref{eq:RA3QAdversary_UpdateRule} converge to optimal assuming the policy for the adversary (or protagonist) is fixed (thus, it is no longer a multi-agent learning setting).

On the side of EGT analysis, previous works used average as payoff~\cite{tuyls2020bounds} and our work considers a risk-averse measure based on variance (second moment), studying higher moments and other measures is one interesting open question.

\section{Conclusions}

We have proposed 4 different Q-learning style algorithms that augment reinforcement learning agents with risk-awareness, variance reduction, and robustness. \RAA and \RAAV are risk-averse but use slightly different techniques to reduce variance. \RAAM and \RAAA are two proposals that extend by adding an adversarial learning layer which is expected to improve its robustness. On the one side, our theoretical results show convergence results for \RAA and \RAAM, on the other side, in our empirical results \RAAV and \RAAA obtained better results in a simplified trading scenario. Lastly, we contributed with risk-averse analysis of our algorithms using empirical game theory. As future work we want to perform a more extensive set of experiments to evaluate the algorithms under different conditions.

\bibliography{all}
\bibliographystyle{icml2021}

\appendix
\onecolumn
\input{appendix.tex}

\end{document}

%% file: math_commands.tex
\usepackage{amsmath,amsfonts,bm}
\usepackage{amssymb,amsthm}
\usepackage{mathtools}
\usepackage{algorithm,algorithmic}
\usepackage{natbib}
\usepackage{xcolor}
\usepackage{url}
\usepackage{cancel}
\usepackage[capitalize]{cleveref}
\usepackage{setspace}
\usepackage{enumitem}

\crefname{proposition}{Proposition}{Propositions}
\crefname{theorem}{Theorem}{Theorems}
\crefname{lemma}{Lemma}{Lemmas}
\crefname{update_rule}{Update}{Updates}
\crefname{algorithm}{Algorithm}{Algorithms}

\def\eqref#1{equation~\ref{#1}}

\def\1{\bm{1}}

\DeclareMathAlphabet{\mathsfit}{\encodingdefault}{\sfdefault}{m}{sl}
\SetMathAlphabet{\mathsfit}{bold}{\encodingdefault}{\sfdefault}{bx}{n}

\def\gA{{\mathcal{A}}}

\def\gF{{\mathcal{F}}}

\def\gP{{\mathcal{P}}}

\def\gS{{\mathcal{S}}}
\def\gT{{\mathcal{T}}}
\def\gU{{\mathcal{U}}}

\def\sP{{\mathbb{P}}}

\def\sR{{\mathbb{R}}}

\def\sV{{\mathbb{V}}}

\newtheorem{theorem}{Theorem}
\newtheorem{lemma}{Lemma}
\newtheorem{definition}{Definition}
\newtheorem{proposition}{Proposition}

\newtheorem{assumption}{Assumption}

\DeclareMathOperator*{\expectation}{\mathbb{E}}

%% file: appendix.tex
\renewcommand{\thesubsection}{\Alph{subsection}}
\renewcommand{\thelemma}{\Alph{subsection}.\arabic{lemma}}
\renewcommand{\theproposition}{\Alph{subsection}.\arabic{proposition}}
\renewcommand{\theassumption}{\Alph{subsection}.\arabic{assumption}}
\begin{center}
\LARGE \textbf{Appendix}
\end{center}

\subsection{Risk-Averse Q-Learning (RAQL) and proof of convergence}

\begin{algorithm}[ht]
\scriptsize
\caption{Risk-Averse Q-Learning (RAQL)~\cite{shen2014risk}}
\label{alg:Risk_Averse_QLearning}
\begin{algorithmic}[1]
\STATE For $\forall (s,a)$, initialize $Q(s,a) = 0$; $N(s,a) = 0$.
\FOR{$t=1$ to $T$}
\STATE At state $s_t$, choose action according to the $\epsilon$-greedy strategy.
\STATE Observe $s_t, a_t, r_t, s_{t+1}$
\STATE $N(s_t, a_t) = N(s_t, a_t) +1$
\STATE Set learning rate $\alpha_t = \frac{1}{N(s_t, a_t)}$
\STATE Update Q : 
\begin{align}\label{eq:Utility_Update_Rule}
    &Q_{t+1}(s_t, a_t) = Q_{t}(s_t, a_t) + \alpha_t (s_t, a_t)\cdot \left[u\left(r_t + \gamma~\cdot \underset{a}{\max}Q_t(s_{t+1},a) - Q_t (s_t, a_t)\right) -x_0\right]
\end{align}
where $u$ is a utility function, here we use $u(x) = -e^{\beta x}$ where $\beta<0$; $x_0 = -1$
\ENDFOR 
\STATE \textbf{Return} Q.
\end{algorithmic}
\end{algorithm}
\subsubsection{Proof of \cref{thm:RARL_Converge}}\label{sec:proof_RARL}
This proof is originally proved by \cite{shen2014risk}, but we describe it here in detail  because it will be useful for later proofs for our proposed algorithms.

First, we show the following Lemma : \begin{lemma}\label{lem:Convergence_of_RiskAverse_UpdateRule}For the iterative procedure \begin{align}
    Q_{t+1}(s_t, a_t) = Q_t(s_t,a_t)+\alpha_t(s_t,a_t)\left[u\left(r_t+\gamma\cdot\underset{a}{\max}Q_t(s_{t+1},a)-Q_t(s_t,a_t)\right)-x_0\right]
\end{align}
where $\alpha_t\geq 0$ satisfy that for any $(s,a)$, $\sum_{t=0}^{\infty}\alpha_t(s,a)=\infty$; and $\sum_{t=0}^{\infty}\alpha^2_t(s,a)<\infty$, then $Q_t\rightarrow Q^*$, where $Q^*$ is the solution of the Bellman equation\begin{align}
    (H^A Q^*)(s,a) = \alpha\cdot\expectation_{s,a}\left[\tilde{u}\left(r_t+\gamma\cdot\underset{a}{\max}Q^*(s_{t+1},a)-Q^*(s,a)\right)\right]+Q^*(s,a) = Q^*(s,a)\qquad \forall (s,a)
\end{align}
\end{lemma}

If \cref{lem:Convergence_of_RiskAverse_UpdateRule} holds, then it's shown in \cite{shen2014risk} that the corresponding policy optimizes the objective function \cref{eq:Risk_Averse_Objective}.

\subsubsection{Proof of \cref{lem:Convergence_of_RiskAverse_UpdateRule}}\label{sec:proof_of_update_operator_converge}
Before proving the convergence, we consider a more general update rule \begin{align}\label{eq:update_rule_1}
    q_{t+1}(i) = (1-\alpha_t (i))q_t(i) + \alpha_t(i)\left[(Hq_t)(i) + w_t(i)\right]
\end{align} 
where $i$ is the independent variable (e.g., in single agent Q learning, it's the state-action pair $(s,a)$), $q_t\in\sR^d$, $H:\sR^d \rightarrow \sR^d$ is an operator, $w_t$ denotes some random noise term and $\alpha_t$ is learning rate with the understanding that $\alpha_t(i) = 0$ if $q(i)$ is not updated at time $t$. Denote by $\gF_t$ the history of the algorithm up to time $t$,\begin{align}
    \gF_t = \{q_0(i),...,q_t(i),w_0(i),...,w_{t}(i),\alpha_0(i),...,\alpha_{t}(i)\}
\end{align}

Recall the following essential proposition : \begin{proposition}\cite{Bertsekas2009}\label{prop:convergence_of_contraction}
Let $q_t$ be the sequence generated by the iteration \cref{eq:update_rule_1}, if we assume the following hold : \begin{enumerate}[label=(\alph*)]
    \item The Learning rates $\alpha_t(i)$ satisfy :
    \begin{align}
        \alpha_t (i) \geq 0;\qquad \sum_{t=0}^{\infty}\alpha_t(i) = \infty; \qquad \sum_{t=0}^{\infty}\alpha_t^2(i) < \infty;\quad \forall i
    \end{align}
    \item The noise terms $w_t(i)$ satisfy\begin{enumerate}[label=(\roman*)]
        \item $\expectation[w_t(i)|\gF_t] = 0$ for all $i$ and $t$;
        \item There exist constants $A$ and $B$ such that $\expectation[w_t^2(i)|\gF_t]\leq A+B\left\|q_t\right\|^2$ for some norm $\left\|\cdot\right\|$ on $\sR^d$.
    \end{enumerate}  
    \item The mapping $H$ is a contraction under sup-norm.
\end{enumerate} 
Then $q_t$ converges to the unique solution $q^*$ of the equation $Hq^* = q^*$ with probability 1.
\end{proposition}
In order to apply \cref{prop:convergence_of_contraction}, we reformulate the update rule \cref{eq:RAA_Q_Update} by letting 
\begin{align}
q_{t+1}(s,a) = \left(1-\frac{\alpha_t(s,a)}{\alpha}\right)q_t(s,a)+\frac{\alpha_t(s,a)}{\alpha}[\alpha\cdot u(d_t) -\alpha\cdot x_0 + q_t(s,a)]
\end{align}

where $\tilde{u}(x) := u(x)-x_0$; $d_t := r_t + \gamma\cdot\underset{a}{\max}q_t(s_{t+1},a) - q_t(s,a)$. And we set \begin{align}\label{eq:def_of_contraction}
    (Hq_t)(s,a) &= \alpha\cdot\expectation_{s,a}\left[\tilde{u}\left(r_t+\gamma\cdot\underset{a}{\max}q_t(s_{t+1},a)-q_t(s,a)\right)\right]+q_t(s,a)\\
    w_t(s,a) &= \alpha\cdot\tilde{u}(d_t)-\alpha\cdot\expectation_{s,a}\left[\tilde{u}(r_t+\gamma\cdot\underset{a}{\max}q_t(s^{\prime},a)-q_t(s,a))\right] 
\end{align}
where $s^{\prime}$ is sampled from $\gT [\cdot | s,a]$.

More explicitly, $Hq$ is defined as \begin{align}
    (Hq)(s,a) = \alpha\cdot\sum_{s^{\prime}}\gT[s^{\prime}|s,a]\cdot\tilde{u}\left(r(s,a)+\gamma\cdot\underset{a^{\prime}}{\max}\;q(s^{\prime},a^{\prime})-q(s,a)\right)+q(s,a)
\end{align}
Next, we show that $H$ is a contraction under sup-norm. 

Note that we assume the utility function satisfy :
\begin{assumption}
\label{ass:Utility_Func_Assumption}
\begin{enumerate}
[label=(\roman*)]
    \item The utility function $u$ is strictly increasing and there exists some $y_0\in\sR$ such that $u(y_0) = x_0$.
    \item There exist positive constants $\epsilon, L$ such that $0<\epsilon\leq \frac{u(x)-u(y)}{x-y}\leq L$ for all $x\ne y\in\sR$.
\end{enumerate}
\end{assumption}

Note that \cref{ass:Utility_Func_Assumption} seems to exclude several important types of utility functions like the exponential function $u(x) = exp(c\cdot x)$ since it does not satisfy the global Lipschitz. But this can be solved by a truncation when $x$ is very large and by an approximation when $x$ is very close to 0. For more details see \citeauthor{shen2014risk}~(2014). 

And we also assume that the immediate reward $r_t$ always satisfy a sub-Gaussian tail assumption. This allows the reward to be unbounded, which is closer to practical settings with tail events, for example, in financial markets.
:\begin{assumption}
$r_t$ is uniformly sub-Gaussian over $t$ with variance proxy $\sigma^2$, i.e., \begin{align}
    \expectation[r_t] &= 0\\
    \expectation[exp(c\cdot r_t)]&\le exp\left(\frac{\sigma^2 c^2}{2}\right)\qquad \forall c\in \sR
\end{align}
\label{ass:immediate_reward_bound}
\end{assumption}
The above uniform sub-Gaussian assumption is equivalent to the following form, commonly seen in statistics and machine learning: there exists $C > 0, \alpha$ such that for every $K > 0$ and every $r_t$, we have:
\begin{align}
    \mathbb{P} (|r_t| > K) \leq C e^{-\alpha K^2}
\end{align}

\begin{proposition}\label{prop:H_is_contraction}
Suppose that \cref{ass:Utility_Func_Assumption} and \cref{ass:immediate_reward_bound} hold and $0<\alpha<\min (L^{-1},1)$. Then there exists a real number $\bar{\alpha}\in[0,1)$ such that for all $q,q^{\prime}\in\sR^d$, $\left\|Hq-Hq^{\prime}\right\|_{\infty}\leq \bar{\alpha}\left\|q-q^{\prime}\right\|_{\infty}$.
\end{proposition} 
\begin{proof}
Define $v(s):=\underset{a}{\max}~q(s,a)$ and $v^{\prime}(s):= \underset{a}{\max}~q^{\prime}(s,a)$. Thus, \begin{align}
    |v(s)-v^{\prime}(s)|\leq \underset{s,a}{\max}|q(s,a)-q^{\prime}(s,a)| = \left\|q-q^{\prime}\right\|_{\infty}
\end{align}
By \cref{ass:Utility_Func_Assumption}, and the monotonicity of $\tilde{u}$, there exists a $\xi_{(x,y)}\in[\epsilon, L]$ such that $\tilde{u}(x)-\tilde{u}(y) = \xi_{(x,y)}\cdot(x-y)$. Then we can obtain \begin{align}
    &(Hq)(s,a) - (Hq^{\prime})(s,a) \\
    &= \sum_{s^{\prime}}\gT[s^{\prime}|s,a]\cdot\Big\{\alpha\xi_{(s,a,s^{\prime},q,q^{\prime})}\cdot\left[\gamma v(s^{\prime})-\gamma v^{\prime}(s^{\prime}) -     q(s,a)+q^{\prime}(s,a)\right]+(q(s,a)-q^{\prime}(s,a))\Big\}\\
    &\leq \left(1-\alpha(1-\gamma)\sum_{s^{\prime}}\gT[s^{\prime}|s,a]\cdot\xi_{(s,a,s^{\prime},q,q^{\prime})}\right)\left\|q-q^{\prime}\right\|_{\infty}\\
    &\leq (1-\alpha(1-\gamma)\epsilon)\left\|q-q^{\prime}\right\|_{\infty}
\end{align}
Hence, $\bar{\alpha} = 1-\alpha(1-\gamma)\epsilon$ is the required constant. 
\end{proof}
Now that we've shown the requirements (a) and (c) of \cref{prop:convergence_of_contraction} hold, it remains to check (b). By \cref{eq:def_of_contraction}, $\expectation[w_t(s,a)|\gF_t] = 0$. Next, we prove (b)(ii).
\begin{align}
    \expectation[w_t^2(s,a)|\gF_t] &= \alpha^2\expectation[(\tilde{u}(d_t))^2|\gF_t]-\alpha^2(\expectation[\tilde{u}(d_t)|\gF_t])^2\\
    &\leq \alpha^2\expectation[(\tilde{u}(d_t))^2|\gF_t]
\end{align}
By \cref{ass:immediate_reward_bound}, $\expectation |r_t| < (2\sigma)^{\frac{1}{2}}\Gamma(\frac{1}{2})$, where $\Gamma(\cdot)$ is the Gamma function (see \cite{SubGaussian80} for details). We denote the upper bound for $\expectation[|r_t|]$ as $R_1$. Then $\expectation[|d_t|]\leq R_1+2\left\|q_t\right\|_{\infty}$, due to \cref{ass:Utility_Func_Assumption}, it implies that\begin{align}
    \expectation\left[|\tilde{u}(d_t)-\tilde{u}(0)|\right] \leq \expectation\left[L\cdot d_t\right]\leq L(R_1+2\left\|q_t\right\|_{\infty})
\end{align}
Hence by triangle inequality, \begin{align}
    \expectation[|\tilde{u}(d_t)|]\leq \tilde{u}(0)+LR_1+2L\left\|q_t\right\|_{\infty}
\end{align}
And since \begin{align}
    (a+b)^2 \leq 2 a^2 + 2 b^2\qquad \forall a,b\in\sR
\end{align}
, we have\begin{align}
    (|\tilde{u}(0)|+LR_1 + 2L\left\|q_t\right\|_{\infty})^2\leq 2(|\tilde{u}(0)|+LR_1)^2 + 8L^2\left\|q_t\right\|_{\infty}^2
\end{align}
And since \begin{align}
    \expectation\left[\left(\tilde{u}(d_t)-\tilde{u}(0)\right)^2 | \gF_t\right] &\le\expectation\left[L\cdot d_t^2\right]\\
    &= \expectation\left[L\cdot\left(r_t + \gamma\cdot\underset{a}{\max}q_t(s^{\prime},a)-q_t(s,a)\right)^2\right]\\
    &= \expectation\left[L\cdot\left(r_t^2 + 2r_t\cdot(\gamma\cdot\underset{a}{\max}q_t(s^{\prime},a)-q_t(s,a))+(\gamma\cdot\underset{a}{\max}q_t(s^{\prime},a)-q_t(s,a))^2\right)\right]\\
    &= LR_2 + 2LR_1(1-\gamma)\cdot\left\|q_t\right\|_{\infty} + L(1-\gamma)^2\cdot\left\|q_t\right\|_{\infty}^2
\end{align}
where $R_2$ is the upper bound for $\expectation[r_t^2]$ due to \cref{ass:immediate_reward_bound} ($\expectation[r_t^2]\leq 4\sigma^2\cdot \Gamma(1)$ \cite{SubGaussian80}).

Note that here $\tilde{u}(0)=0$, hence we have \begin{align}
    \alpha^2\expectation[(\tilde{u}(d_t) )^2|\gF_t]\leq \alpha^2\cdot\left(LR_2 + 2LR_1(1-\gamma)\cdot\left\|q_t\right\|_{\infty} + L(1-\gamma)^2\cdot\left\|q_t\right\|_{\infty}^2\right)
\end{align}
Hence, \begin{align}
    \expectation[w_t^2(s,a)|\gF_t]\leq 2\alpha^2\cdot\left(LR_2 + 2LR_1(1-\gamma)\cdot\left\|q_t\right\|_{\infty} + L(1-\gamma)^2\cdot\left\|q_t\right\|_{\infty}^2\right)
\end{align}
if $\left\|q_t\right\|_{\infty}\leq 1$, then \begin{align}
    \expectation[w_t^2(s,a)|\gF_t]\leq 2\alpha^2\cdot\left(LR_2 + 2LR_1(1-\gamma) + L(1-\gamma)^2\cdot\left\|q_t\right\|_{\infty}^2\right)
\end{align}
if $\left\|q_t\right\|_{\infty}> 1$, then \begin{align}
    \expectation[w_t^2(s,a)|\gF_t]\leq 2\alpha^2\cdot\left(LR_2 +(2LR_1(1-\gamma)+L(1-\gamma)^2)\cdot\left\|q_t\right\|_{\infty}^2\right)
\end{align}
Then we have shown that $q_t$ satisfy all of the requirements in \cref{prop:convergence_of_contraction}, then $q_t\rightarrow q^*$ with probability 1. 

\subsection{Nash-Q Learning Algorithm}
\label{sec:Multi_Agent_QLearning}
This section describes the Nash-Q Learning Algorithm~\cite{MultiAgentQLearning98} and its convergence guarantees, we restate them here since our \cref{alg:MultiAgent_QLearning_RiskAverse} (\RAAM) is designed based on Nash-Q. Also note that \cref{ass:Bimatrix_Nash_Assumption} will also be used in \RAAM.
 
\begin{algorithm}[ht]
\caption{Nash Q-Learning for Agent $A$~\cite{MultiAgentQLearning98}}
\label{alg:MultiAgent_QLearning}
\scriptsize
\begin{algorithmic}[1]
\STATE For $\forall (s,a_A,a_B)$, initialize $Q^1_A(s,a_A,a_B) = 0$; $Q^2_A(s,a_A,a_B) = 0$; $N_A(s,a_A,a_B) = 0$.
\FOR{$t = 1$ to $T$}
\STATE At state $s_t$, compute $\pi^1_A(s_t)$, which is a mixed strategy Nash equilibrium solution of the bimatrix game $(Q^1_A(s_t), Q^2_A(s_t))$.
\STATE Choose action $a_t^A$ based on $\pi^1_A(s_t)$ according to $\epsilon$-greedy strategy.
\STATE Observe $r_t^A, r_t^B, a_t^B$ and $s_{t+1}$.
\STATE At state $s_{t+1}$, compute $\pi^1_A(s_{t+1})$,$\pi^2_A(s_{t+1})$, which are mixed strategies Nash equilibrium solution of the bimatrix game $(Q^1_A(s_{t+1}), Q^2_A(s_{t+1}))$.
\STATE $N_A(s_t,a^A_t, a^B_t) = N_A(s_t,a^A_t, a^B_t) + 1$
\STATE Set learning rate $\alpha_t^A = \frac{1}{N_A(s_t,a^A_t, a^B_t)}$.
\STATE Update $Q_A^1, Q_A^2$ such that 
\begin{align*}
    Q_A^1(s_t,a^A_t, a^B_t) = (1-\alpha_t^A)\cdot Q_A^1(s_t,a^A_t, a^B_t) + \alpha_t^A\cdot\left[r_t^A + \gamma\cdot\pi^1_A(s_{t+1})Q_A^1(s_{t+1})\pi^2_A(s_{t+1})\right]\\
    Q_A^2(s_t,a^A_t, a^B_t) = (1-\alpha_t^A)\cdot Q_A^2(s_t,a^A_t, a^B_t) + \alpha_t^A\cdot\left[r_t^B + \gamma\cdot\pi^1_A(s_{t+1})Q_A^2(s_{t+1})\pi^2_A(s_{t+1})\right]
\end{align*}
\ENDFOR
\end{algorithmic}
\end{algorithm}
\begin{assumption}\cite{MultiAgentQLearning98}\label{ass:Bimatrix_Nash_Assumption}
A Nash equilibrium $(\pi^1(s),\pi^2(s))$ for any bimatrix game $(Q^1(s),Q^2(s))$ during the training process satisfies one of the following properties :\begin{enumerate}
    \item The Nash equilibrium is global optimal. \begin{align}
        \pi^1(s)Q^k(s)\pi^2(s)\geq \hat{\pi}^1(s)Q^k(s)\hat{\pi}^2(s)\qquad \forall\hat{\pi}^1 (s),\hat{\pi}^2 (s),\;and\;k=1,2
    \end{align}
    \item If the Nash equilibrium is not a global optimal, then an agent receives a higher payoff when the other agent deviates from the Nash equilibrium strategy.\begin{align}
        \pi^1(s)Q^1(s)\pi^2(s)\leq \pi^1(s)Q^1(s)\hat{\pi}^2(s)\qquad \forall \hat{\pi}^2(s)\\
        \pi^1(s)Q^2(s)\pi^2(s)\leq \hat{\pi}^1(s)Q^2(s)\pi^2(s)\qquad \forall \hat{\pi}^1(s)
    \end{align}
\end{enumerate}
\end{assumption}
\begin{theorem} (Theorem 4, \citeauthor{MultiAgentQLearning98}~1998)
Under \cref{ass:Bimatrix_Nash_Assumption}, the coupled sequences $Q_A^1, Q_A^2$ updated by \cref{alg:MultiAgent_QLearning}, converge to the Nash equilibrium Q values $(Q^{1}_{*}, Q^{2}_{*})$, with $Q^{k}_{*}\;(k=1,2)$ defined as 
\begin{align}
    Q^{1}_*(s,a^A,a^B) = r^A(s,a^A,a^B) + \gamma\cdot\expectation_{s^{\prime}\sim\gP(\cdot|s,a^A,a^B)}\left[J^A(s^{\prime},\pi^{A}_{*}, \pi^{B}_{*})\right]\\
    Q^{2}_*(s,a^A,a^B) = r^B(s,a^A,a^B) + \gamma\cdot\expectation_{s^{\prime}\sim\gP(\cdot|s,a^A,a^B)}\left[J^B(s^{\prime},\pi^{A}_{*}, \pi^{B}_{*})\right]
\end{align}
where $(\pi^{A}_{*}, \pi^{B}_{*})$ is a Nash equilibrium solution for this stochastic game $(J^A, J^B)$ and \begin{align}
    J^A(s^{\prime},\pi^{A}_{*}, \pi^{B}_{*}) = \sum_{t=0}^{\infty}\gamma^t\expectation\left[r_t^A|\pi^A_*, \pi^B_*, s_0 = s^{\prime}\right]\\
    J^B(s^{\prime},\pi^{A}_{*}, \pi^{B}_{*}) = \sum_{t=0}^{\infty}\gamma^t\expectation\left[r_t^B|\pi^A_*, \pi^B_*, s_0 = s^{\prime}\right]
\end{align}
\end{theorem}
\subsection{Proof of \cref{thm:Convergence_of_RAA}}\label{sec:proof_of_convergence_RAA}
Poisson masks $M\sim Poisson(1)$ provides parallel learning since $Binomial(T, \frac{1}{T})\rightarrow Poisson(1)$ as $T\rightarrow\infty$, so each Q table $Q^i$ is trained in parallel. The proof of convergence of $Q^i$ for all $i\in\{1,..., k\}$ is exactly same as \cref{sec:proof_RARL}. Hence $\frac{1}{k}\sum_{i=1}^{k}Q^i\rightarrow Q^*$ w.p. 1.

\subsection{Proof of convergence of \cref{alg:MultiAgent_QLearning_RiskAverse} (\RAAM)}
In this section, we prove the convergence of \cref{alg:MultiAgent_QLearning_RiskAverse} under \cref{ass:Bimatrix_Nash_Assumption}.

The convergence proof is based on the following lemma\begin{lemma}\label{lem:Conditional_Averate_lemma}[Conditional Averaging Lemma~\cite{ValueBasedRL99}] Assume the learning rate $\alpha_t$ satisfies \cref{prop:convergence_of_contraction}(a). Then, the process 
$Q_{t+1}(i) =(1-\alpha_t(i))Q_t(i)+\alpha_t w_t(i)$ converges to $\expectation[w_t(i)|h_t, \alpha_t]$, where $h_t$ is the history at time $t$.
\end{lemma}
We take the proof of convergence of $Q^P$ as an example, and the proof of convergence of $Q^A$ is exactly the same. And we first reformulate the update rule \cref{eq:UpdateRule_MultiAgent_RiskAverse_1} as :\begin{align}
    &Q^P(s_t,a^P_t, a^A_t) = (1-\frac{\alpha_t}{\alpha})\cdot Q^P(s_t,a^P_t, a^A_t) +\\ &\frac{\alpha_t}{\alpha}\cdot \left[\alpha\cdot u^P\left(r_t^P + \gamma\cdot\pi^P(s_{t+1})Q^P(s_{t+1})\pi^A(s_{t+1})-Q^P(s_t,a^P_t, a^A_t)\right) - \alpha\cdot x_0 + Q^P(s_t,a^P_t, a^A_t)\right]
\end{align}
And we set \begin{align}\label{eq:def_of_wt_2}
    (H^P Q^P)(s_t,a^P_t, a^A_t) &= \alpha\cdot u^P\left(r_t^P + \gamma\cdot\pi^P(s_{t+1})Q^P(s_{t+1})\pi^A(s_{t+1})-Q^P(s_t,a^P_t, a^A_t)\right) - \alpha\cdot x_0 + Q^P(s_t,a^P_t, a^A_t)
\end{align}
And $H^A Q^A$ is defined symmetrically as \begin{align}
    (H^A Q^A)(s_t,a^P_t, a^A_t) &= \alpha\cdot u^A\left(r_t^A + \gamma\cdot\pi^P(s_{t+1})Q^A(s_{t+1})\pi^A(s_{t+1})-Q^A(s_t,a^P_t, a^A_t)\right) - \alpha\cdot x_1 + Q^A(s_t,a^P_t, a^A_t)
\end{align}
It's shown in \cite{MultiAgentQLearning98} that the operator $(M^P_t,M^A_t)$ is a $\gamma$-contraction mapping where $(M^P_t,M^A_t)$ is defined as \begin{align}
    M^P_t Q^P (s) = r^P_t + \gamma\cdot\pi^P(s)Q^P(s)\pi^A(s)\\
    M^A_t Q^A (s) = r^A_t + \gamma\cdot\pi^P(s)Q^A(s)\pi^A(s)
\end{align}

Next, we show that $(H^P, H^A)$ is a contraction under sup-norm (under assumption \cref{ass:Utility_Func_Assumption}).\begin{align}
    H^P Q^P - H^P \hat{Q}^P &=\alpha\cdot\left[\xi^P_{Q^P,\hat{Q}^P}\cdot\left(M^P Q^P - M^P \hat{Q}^P -(Q^P-\hat{Q}^P)\right)\right] + (Q^P-\hat{Q}^P)\\
    &\le \alpha\cdot\left[\xi^P_{Q^P,\hat{Q}^P}\cdot(\gamma-1)\left\|Q^P-\hat{Q}^P\right\|_{\infty}\right] + \left\|Q^P-\hat{Q}^P\right\|_{\infty}\\
    &\leq \left(1-\alpha\epsilon(1-\gamma)\right)\cdot\left\|Q^P-\hat{Q}^P\right\|_{\infty}
\end{align}
Similarly, $ H^A Q^A - H^A \hat{Q}^A\leq \left(1-\alpha\epsilon(1-\gamma)\right)\cdot\left\|Q^A-\hat{Q}^A\right\|_{\infty}$.

Hence $(H^P, H^A)$ is a $\left(1-\alpha\epsilon(1-\gamma)\right)$-contraction under sup-norm. Hence by \cref{lem:Conditional_Averate_lemma} the update rule \cref{eq:UpdateRule_MultiAgent_RiskAverse_1,eq:UpdateRule_MultiAgent_RiskAverse_2} respectively converges to \begin{align}
    Q^P(s_t,a^P_t, a^A_t)\rightarrow\expectation\left[\alpha\cdot u^P\left(r_t^P + \gamma\cdot\pi^P(s_{t+1})Q^P(s_{t+1})\pi^A(s_{t+1})-Q^P(s_t,a^P_t, a^A_t)\right) - \alpha\cdot x_0 + Q^P(s_t,a^P_t, a^A_t)\right]\\
    Q^A(s_t,a^P_t, a^A_t)\rightarrow\expectation\left[\alpha\cdot u^A\left(r_t^A + \gamma\cdot\pi^P(s_{t+1})Q^A(s_{t+1})\pi^A(s_{t+1})-Q^A(s_t,a^P_t, a^A_t)\right) - \alpha\cdot x_1 + Q^A(s_t,a^P_t, a^A_t)\right]
\end{align}

i.e., \cref{eq:UpdateRule_MultiAgent_RiskAverse_1,eq:UpdateRule_MultiAgent_RiskAverse_2} respectively converges to $Q^*_P,Q^*_A$ where $Q^*_P, Q^*_A$ are the solution to the Bellman equations \begin{align}\label{eq:Bellman_Q_MultiAgent_1}
    \expectation_{s,a^P,a^A}\left[u^P\left(r^P(s,a^P,a^A) + \gamma\cdot\pi^{P*}(s^{\prime})Q_P^*(s^{\prime})\pi^{A*}(s^{\prime}) - Q_P^*(s,a^P,a^A)\right)\right] = x_0\\
    \expectation_{s,a^P,a^A}\left[u^A\left(r^A(s,a^P,a^A) + \gamma\cdot\pi^{P*}(s^{\prime})Q_A^*(s^{\prime})\pi^{A*}(s^{\prime}) - Q_A^*(s,a^P,a^A)\right)\right] = x_1
\end{align}

where $(\pi^{P*},\pi^{A*})$ is the Nash equilibrium solution to the bimatrix game $(Q_P^*, Q_A^*)$. Next we show that $(\pi^{P*},\pi^{A*})$ is a Nash equilibrium solution for the game with equilibrium payoffs $\left(\tilde{J}^P(s,\pi^{P*},\pi^{A*}), \tilde{J}^A(s,\pi^{P*},\pi^{A*})\right)$. 

As in \cite{shen2014risk}, for any $X\in\sR$, define $\gU^P(X|s,a^P,a^A) :\sR \times \gS\times\gA\times\gA\rightarrow\sR$ be a mapping (for brevity, could be written as $\gU^P_{s,a^P,a^A}(X)$ ) defined by \begin{align}
    \gU^P_{s,a^P,a^A}(X) = sup \Big\{m\in\sR|\expectation_{s,a^P,a^A}\left[u^P(X-m)\right]\geq x_0\Big\}
\end{align}
Similar to \cite{shen2014risk,RiskSensitiveshen13}, suppose $(\pi^P,\pi^A)$ is a Nash equilibrium solution to the game $\left(\tilde{J}^P(s,\pi^P,\pi^A), \tilde{J}^A(s,\pi^P,\pi^A)\right)$, then the payoffs $\tilde{J}^P(s,\pi^P,\pi^A),\; \tilde{J}^A(s,\pi^P,\pi^A)$ are the solution to the risk-sensitive Bellman equations\begin{align}\label{eq:State_Value_Optimize}
    \tilde{J}^P(s,\pi^P,\pi^A) = \pi^P(s)\gU^P_{s,a^P,a^A}\left(r^P(s,:,:) + \gamma\cdot \tilde{J}^P(s^{\prime},\pi^P,\pi^A)\right)\pi^A(s)\qquad\forall s\in\gS\\
    \tilde{J}^A(s,\pi^P,\pi^A) = \pi^P(s)\gU^P_{s,a^P,a^A}\left(r^A(s,:,:) + \gamma\cdot \tilde{J}^A(s^{\prime},\pi^P,\pi^A)\right)\pi^A(s)\qquad\forall s\in\gS
\end{align}
And the corresponding Q tables satisfies\begin{align}\label{eq:Q_Value_Optimize_Equilibrium}
    Q_P (s,a^P,a^A) = \gU^P_{s,a^P,a^A}\left(r^P(s,a^P,a^A) + \gamma \tilde{J}^P(s^{\prime},\pi^P,\pi^A)\right)\\
    Q_A (s,a^P,a^A) = \gU^P_{s,a^P,a^A}\left(r^A(s,a^P,a^A) + \gamma \tilde{J}^A(s^{\prime},\pi^P,\pi^A)\right)
\end{align} 
Note that $\gU^P_{s,a^P,a^A}$ is monotonic one-to-one mapping, so as shown in [\textbf{Theorem 4.6.5}~\cite{filar-competitiveMDP}], $(\pi^P, \pi^A)$ are the Nash equilibrium solution to the bimatrix game $(Q_P, Q_A)$.
Then if we can show that $Q_P = Q_P^*$ and $Q_A = Q_A^*$ (i.e., $Q_P$ and $Q_A$ are the solution to \cref{eq:Bellman_Q_MultiAgent_1} ), then the Nash solution of the bimatrix game $(Q_P^*,Q_A^*)$ returned by \cref{alg:MultiAgent_QLearning_RiskAverse} will be the Nash solution for the game $(\tilde{J}^P, \tilde{J}^A)$. 

\cite{shen2014risk} showed that \cref{eq:Q_Value_Optimize_Equilibrium} is equivalent to \begin{align}
    \expectation_{s,a^P,a^A}\left[u^P\left(r^P(s,a^P,a^A) + \gamma \tilde{J}^P(s^{\prime},\pi^P,\pi^A) - Q_P (s,a^P,a^A)\right)\right] = x_0 \\
    \expectation_{s,a^P,a^A}\left[u^A\left(r^A(s,a^P,a^A) + \gamma \tilde{J}^A(s^{\prime},\pi^P,\pi^A) - Q_A (s,a^P,a^A)\right)\right] = x_1
\end{align}
Plugging \cref{eq:State_Value_Optimize} in, we get \begin{align}
    \expectation_{s,a^P,a^A}\left[u^P\left(r^P(s,a^P,a^A) + \gamma\cdot \pi^P Q_P(s^{\prime})\pi^A - Q_P (s,a^P,a^A)\right)\right] = x_0 \\
    \expectation_{s,a^P,a^A}\left[u^A\left(r^A(s,a^P,a^A) + \gamma\cdot \pi^P Q_A(s^{\prime})\pi^A - Q_A (s,a^P,a^A)\right)\right] = x_1
\end{align}
which is exactly \cref{eq:Bellman_Q_MultiAgent_1}.

Hence we have shown that under \cref{ass:Bimatrix_Nash_Assumption}, \cref{eq:State_Value_Optimize} and \cref{eq:Bellman_Q_MultiAgent_1} are equivalent. Hence \cref{alg:MultiAgent_QLearning_RiskAverse} converges to $(Q_P^*, Q_A^*)$ s.t. the Nash equilibrium solution $(\pi^{P*}, \pi^{A*})$ for the bimatrix game $(Q_P^*, Q_A^*)$ is the Nash equilibrium solution to the game and the equilibrium payoffs are $\tilde{J}^P(s,\pi^{P*}, \pi^{A*})$; $\tilde{J}^A(s,\pi^{P*}, \pi^{A*})$.

\subsection{Discussion of RA3-Q}
\label{sec:Discussion_of_RAAA}

\begin{algorithm*}[t!]
\footnotesize
\caption{Risk-Averse Adversarial Averaged Q-Learning (\RAAA)}
\label{alg:Risk_Averse_Adversarial_Averaged_QLearning_fullversion}
\textbf{Input :} Training steps $T$; Exploration rate $\epsilon$; Number of models $k$; Risk control parameters $\lambda_P, \lambda_A$; Utility function parameters $\beta^P < 0; \beta^A > 0$. 
\begin{spacing}{0.8}
\begin{algorithmic}[1]
\STATE Initialize $Q_P^i(s,a_P,a_A)= 0$; $Q_A^i(s,a_P,a_A)= 0$ for $\forall i = 1,..., k \;$and$\;(s,a_A, a_P)$; $N= \mathbf{0}\in\sR^{|\gS|\times|\gA|\times|\gA|}$;
\STATE Randomly sample action choosing head integers $H_P, H_A\in\{1,...,k\}$.
\FOR{$t = 1$ to $T$}
\STATE $Q_{P} = Q_{P}^{H_P}$
\STATE Compute $\hat{Q}_{P}$ by
\begin{align}
    \hat{Q}_{P}(s,a_P,a_A) = Q_{P}(s,a_P,a_A) - \lambda_P\cdot \frac{\sum_{i=1}^{k}(Q_P^i(s,a_P,a_A) - \bar{Q}_P(s,a_P,a_A))^2}{k-1} \qquad \lambda_P>0
\end{align}
where $\bar{Q}_P(s,a_P,a_A) = \frac{1}{k}\sum_{i=1}^{k}Q_P^i(s,a_P,a_A)$

\STATE $Q_{A} = Q_{A}^{H_A}$
\STATE Compute $\hat{Q}_{A}$ by \begin{align}
    \hat{Q}_{A}(s,a_P,a_A) = Q_{A}(s,a_P,a_A) + \lambda_A\cdot \frac{\sum_{i=1}^{k}(Q_A^i(s,a_P,a_A) - \bar{Q}_A(s,a_P,a_A))^2}{k-1}\qquad \lambda_A>0
\end{align}
where $\bar{Q}_A(s,a_P,a_A) = \frac{1}{k}\sum_{i=1}^{k}Q_A^i(s,a_P,a_A)$
\STATE The optimal actions $(a_P^{\prime}, a_A^{\prime})$ are defined as \begin{align}
    \hat{Q}_P(s_t, a_P^{\prime}, a_A^0) = \underset{a_P,a_A}{\max}\hat{Q}_P(s_t,a_P,a_A)\qquad\text{for some $a_A^0$}\\
    \hat{Q}_A(s_t, a_P^0,  a_A^{\prime}) = \underset{a_P,a_A}{\max}\hat{Q}_A(s_t,a_P,a_A)\qquad\text{for some $a_P^0$}
\end{align}
\STATE Select actions $a_P, a_A$ according to $\hat{Q}_{P},\hat{Q}_{A}$ by applying $\epsilon$-greedy strategy.
\STATE Two agents respectively execute actions $a_P,a_A$ and observe $(s_t, a_P,a_A, r^A_t, r^P_t, s_{t+1})$ 
\STATE Generate mask $M\in \sR^k \sim Poisson(1)$
\STATE $N(s_t,a_P,a_A)= N(s_t,a_P,a_A) + 1$
\STATE $\alpha(s_t,a_P,a_A) =\frac{1}{N(s_t,a_P,a_A)}$
\FOR{$i=1,...,k$}
\IF{$M_i = 1$}
\STATE Update $Q_P^i$ by \begin{align}\label{eq:RA3QProtagonist_UpdateRule}
    Q_P^i(s_t,a_P, a_A) = Q_P^i(s_t,a_P,a_A) + \alpha(s_t,a_P,a_A)\cdot\left[u^P\left(r^P_t +  \gamma\cdot\underset{a_P,a_A}{\max} Q_P^i(s_{t+1}, a_P,a_A) - Q_P^i(s_t, a_P,a_A)\right)-x_0\right]
\end{align}
where $u^P$ is a utility function, here we use $u^P(x) = -e^{\beta^P x}$ where $\beta^P<0$; $x_0 = -1$
\ENDIF
\ENDFOR
\FOR{$i=1,...,k$}
\IF{$M_i = 1$}
\STATE Update $Q_A^i$ by \begin{align}\label{eq:RA3QAdversary_UpdateRule}
    Q_A^i(s_t,a_P,a_A) = Q_A^i(s_t,a_P,a_A) + \alpha(s_t,a_P,a_A)\cdot\left[u^{A}\left(r^A_t + \gamma\cdot\underset{a_P,a_A}{\max} Q_A^i(s_{t+1}, a_P,a_A) - Q_A^i(s_t,a_P, a_A)\right)-x_1\right]
\end{align}
where $u^A$ is a utility function, here we use $u(x) = e^{\beta^A\cdot x}$ where $\beta^A>0$; $x_1 = 1$
\ENDIF
\ENDFOR
\STATE Update $H_P$ and $H_A$ by randomly sampling integers from $1$ to $k$
\ENDFOR
\STATE \textbf{Return} $\frac{1}{k}\sum_{i=1}^{k}Q_P^i$; $\frac{1}{k}\sum_{i=1}^{k}Q_A^i$
\end{algorithmic}
\end{spacing}
\end{algorithm*}

We have presented a short version of \RAAA in \cref{alg:Risk_Averse_Adversarial_Averaged_QLearning}, a detailed version is presented in
\cref{alg:Risk_Averse_Adversarial_Averaged_QLearning_fullversion}.

In this section, we discuss convergence issues on \RAAA. First we discuss a simplified setting where we show that if the adversary's policy is a \emph{fixed} policy $\pi^A_0$, the update rule for protagonist \cref{eq:RA3QProtagonist_UpdateRule} converges to the optimal of $J^P(s,:,\pi^A_0)$. Similarly, if the protagonist's policy is a \emph{fixed} policy $\pi^P_0$, the update rule for adversary \cref{eq:RA3QAdversary_UpdateRule} converges to the optimal of $J^A(s,\pi^P_0,:)$.

Poisson masks $M\sim Poisson(1)$ provides parallel learning since $Binomial(T, \frac{1}{T})\rightarrow Poisson(1)$ as $T\rightarrow\infty$, so each Q table of protagonist/adversary, $Q^i_P$, $Q^i_A$, are trained in parallel respectively.

Similar to \cref{sec:proof_RARL}, we need to prove the convergence of the iterative procedure. We take agent protagonist as an example, and the proof for adversary is similar. 

Fix the policy for adversary, then according to [\cite{shen2014risk} \textbf{Proposition 3.1}], for any random variable $X$, the following statements are equivalent 
$$\text{(i) } \frac{1}{\beta^P}\log \expectation_{\mu}\left[exp\left(\beta^P\cdot X\right)\right] = m^* $$

$$\text{(ii) } \expectation_{\mu}\left[u^P(X-m^*)\right] = x_0$$

We'll use this proposition in the following context to show that our convergent point is the optimal of the objective function $\tilde{J}^P(s,:,\pi^A_0)$.

Compared to \cref{alg:Risk_Averse_QLearning} (RAQL),  \RAAA uses multi-agent extension of MDP (where the transition function is $\gP:\gS\times\gA\times\gA\rightarrow\sR^{\gS}$. We reformulate the update rule \cref{eq:RA3QProtagonist_UpdateRule} by letting \begin{align}
    &q_{t+1}^{P}(s,a_P,a_A) = \left(1- \frac{\alpha_t(s,a_P,a_A)}{\alpha}\right)q^P_t(s,a_P,a_A) + \frac{\alpha_t(s,a_P,a_A)}{\alpha}\cdot\left[\alpha\cdot u(d_t) - x_0 + q^P_t(s,a_P,a_A) \right]\\
    &\text{where } d_t := r_t^P + \gamma\cdot\underset{a_P,a_A}{\max}\;q^P_t(s^{\prime},a_P,a_A) - q^P_t(s,a_P,a_A)\qquad x_0 = -1\qquad \alpha\in(0,\min(L^{-1},1)]
\end{align}
And we set \begin{align}
    (H^P q_t^P)(s,a_P,a_A) &= \alpha\cdot\expectation_{s,a_P,a_A}\left[\tilde{u}\left(r_t^P  + \gamma\cdot\underset{a_P,a_A}{\max}\;q^P_t(s^{\prime},a_P,a_A) - q^P_t(s,a_P,a_A)\right)\right] + q_t^P(s,a_P,a_A)\\
    w_t(s,a_P,a_A) &=\alpha\cdot\tilde{u}(d_t) - \alpha\cdot\expectation_{s,a_P,a_A}\left[\tilde{u}\left(r_t^P  + \gamma\cdot\underset{a_P,a_A}{\max}\;q^P_t(s^{\prime},a_P,a_A) - q^P_t(s,a_P,a_A)\right)\right]\label{eq:def_of_wt}\\
    \tilde{u}(x) &= u(x) - x_0
\end{align}
Next we show that $H^P$ is a $(1-\alpha(1-\gamma)\epsilon)$-contractor under \cref{ass:Utility_Func_Assumption}:

For any two q tables $q,q^{\prime}$, define $v^P(s):= \underset{a_P,a_A}{\max}\;q(s,a_P,a_A)$ and $v^{P^\prime}(s):= \underset{a_P,a_A}{\max}\;q^{\prime}(s,a_P,a_A)$. Thus,\begin{align}
    |v^{P}(s)-v^{P^\prime}(s)|\leq \underset{s,a_P,a_A}{\max}|q(s,a_P,a_A) - q^{\prime}(s,a_P,a_A)| = \left\|q-q^{\prime}\right\|_{\infty}\;
\end{align}
By \cref{ass:Utility_Func_Assumption} and monotonicity of $\tilde{u}$, for given $x,y\in\sR$, there exists $\xi_{(x,y)}\in[\epsilon,L]$ such that $$\tilde{u}(x) - \tilde{u}(y) = \xi_{(x,y)}\cdot (x-y).$$ 
Then we can obtain
\begin{align}
    &(H^P q)(s,a_P,a_A) - (H^P q^{\prime})(s,a_P,a_A)\\
    &= \sum_{s^{\prime}}\gP[s^{\prime}|s,a_P,a_A]\cdot\Big\{\alpha\xi_{(s,a_P,a_A,s^{\prime},q,q^{\prime})}\cdot[\gamma\cdot v^{P}(s^{\prime}) - \gamma\cdot v^{P^{\prime}}(s^{\prime}) - q(s,a_P,a_A) + q^{\prime}(s,a_P,a_A)] + (q(s,a_P,a_A) - q^{\prime}(s,a_P,a_A))\Big\}\\
    &\leq \left(1-\alpha(1-\gamma)\sum_{s^{\prime}}\gP[s^{\prime}|s,a_P,a_A]\cdot\xi_{(s,a_P,a_A,s^{\prime},q,q^{\prime})}\right)\left\|q-q^{\prime}\right\|_{\infty}\\
    &\leq (1-\alpha(1-\gamma)\epsilon)\left\|q-q^{\prime}\right\|_{\infty}
\end{align}
Hence $H^P$ is a contractor.

By \cref{eq:def_of_wt}, $\expectation\left[w_t(s,a_P,a_A)|\gF_t\right] = 0$. Hence it remains to prove b(ii) in \cref{prop:convergence_of_contraction}. \begin{align}
    \expectation\left[w_t^2(s,a_P,a_A)|\gF_t\right] = \alpha^2\cdot\expectation\left[(\tilde{u}(d_t))^2|\gF_t\right] - \alpha^2(\expectation\left[\tilde{u}(d_t)|\gF_t\right])^2\leq \alpha^2\cdot\expectation\left[(\tilde{u}(d_t))^2|\gF_t\right]
\end{align}
Following from the same procedures as \cref{sec:proof_RARL}, condition b(ii) of \cref{prop:convergence_of_contraction} also holds in this case. And recall that the learning rate satisfies condition a, hence by \cref{prop:convergence_of_contraction}, $q\rightarrow q^*$, where $q^*$ is the solution to the Bellman equation \begin{align}
    \expectation_{s,a_P,a_A}\left[u^P\left(r_t^P + \gamma\cdot\underset{a_P,a_A}{\max}\;q(s^{\prime},a_P,a_A) - q(s,a_P,a_A)\right)\right] = x_0\qquad \pi^A_0\text{ is fixed}
\end{align}
for $\forall (s,a_P,a_A)$. Where $s^{\prime}$ is sampled from $\gP[\cdot|s,a_P,a_A]$. Similarly, we can show that for a fixed policy for protagonist, the update rule \cref{eq:RA3QAdversary_UpdateRule} will guarantee that $q_A\rightarrow q_A^*$, where $q_A^*$ is the solution to the Bellman equation \begin{align}
    \expectation_{s,a_P,a_A}\left[u^A\left(r^A_t + \gamma\cdot\underset{a_P,a_A}{\max}\;q(s^{\prime},a_P,a_A) - q(s,a_P,a_A)\right)\right] = x_1 \qquad \pi^P_0\text{ is fixed}
\end{align}
for $\forall (s,a_P,a_A)$. Where $s^{\prime}$ is sampled from $\gP[\cdot|s,a_P,a_A]$. 

Note that this does not imply a convergence guarantee of \RAAA because of the \emph{protagonist/adversary's policy is fixed} assumption. Only if one of the agents (say protagonist) stops learning (and its policy becomes fixed) at some point, then the other agent (adversary) will also converge. Note that in the general multi-agent learning case this is always a challenge and it is often hard to a balance between theoretical algorithms (with convergence guarantees) and practical algorithms (loosing guarantees but with good empirical results), see our experimental results in \cref{sec:risk_and_robustness_evaluation} and related literature~\cite{bowling2002multiagent,weinberg2004best,littman2001value}.

\subsection{Meta-game payoff examples and EGT plots}
\label{sec:meta_game_examples}
\begin{table}[h!]
\scriptsize
    \caption{Payoff Table of Rock-Paper-Scissors}
    \label{table:Payoff_RPS}
    \centering
    \begin{tabular}{c c c|c c c}
        \toprule
       $N_{Rock}$  & $N_{Paper}$ & $N_{Scissors}$ & $R_{Rock}$ & $R_{Paper}$ & $R_{Scissors}$\\ \hline
       2 & 0 & 0 & 0 & 0 & 0\\
       1  &  1 & 0 & -1 & 1 & 0\\
       0 & 2 & 0 & 0 & 0 & 0 \\
       1 & 0 & 1 & 1 & 0 & -1\\
       0 & 0 & 2 & 0 & 0 & 0\\
       0 & 1 & 1 & 0 & -1 & 1\\
       \bottomrule
    \end{tabular}
\end{table}

\begin{figure}[!h]
    \centering
    \includegraphics[width=5cm]{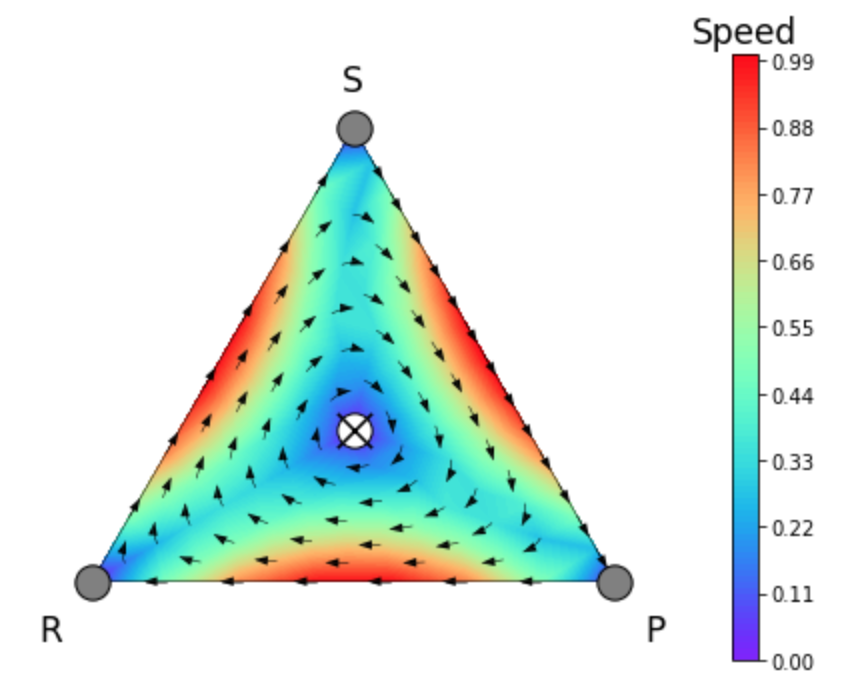}
    \caption{Directional Field of Rock-Paper-Scissors}
    \label{fig:directional_example_rps}
\end{figure}
\begin{figure}[!h]
    \centering
    \includegraphics[width=5cm]{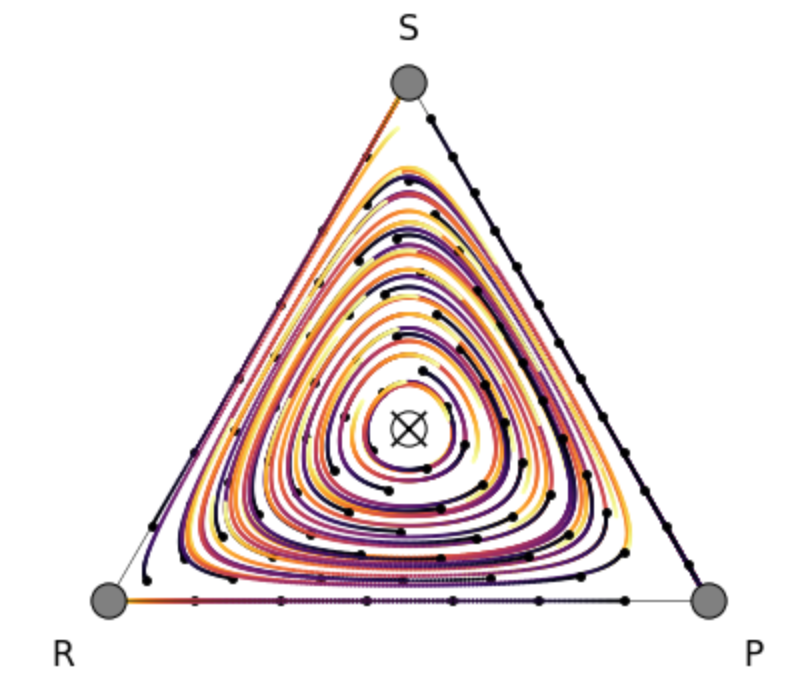}
    \caption{Trajectory Plot of Rock-Paper-Scissors}
    \label{fig:trajectory_example_rps}
\end{figure}
The payoff table of a well-known game rock-scissors-papers is as shown in \cref{table:Payoff_RPS}, its corresponding directional field is as shown in \cref{fig:directional_example_rps}, and its trajectory plot is as shown in \cref{fig:trajectory_example_rps}. It can be observed from \cref{fig:directional_example_rps,fig:trajectory_example_rps} that the equilibrium of Rock-Paper-Scissors is the centroid of the strategies simplex.
\begin{table}[h!]
\scriptsize
    \caption{An example of a meta game payoff table of 2 players, 3 strategies.}
    \label{table:MetaPayoff}
    \centering
    \begin{tabular}{c c c|c c c}
        \toprule
       $N_{i1}$  & $N_{i2}$ & $N_{i3}$ & $R_{i1}$ & $R_{i2}$ & $R_{i3}$\\ \hline
       2 & 0 & 0 & 0.5 & 0 & 0\\
       1  &  1 & 0 & 0.3 & 0.7 & 0\\
       0 & 2 & 0 & 0 & 0.9 & 0 \\
       1 & 0 & 1 & 0.35 & 0 & 0.45\\
       0 & 0 & 2 & 0 & 0 & 0.6\\
       0 & 1 & 1 & 0 & 0.66 & 0.38\\
       \bottomrule
    \end{tabular}
\end{table}

\begin{figure}[!h]
    \centering
    \includegraphics[width=5cm]{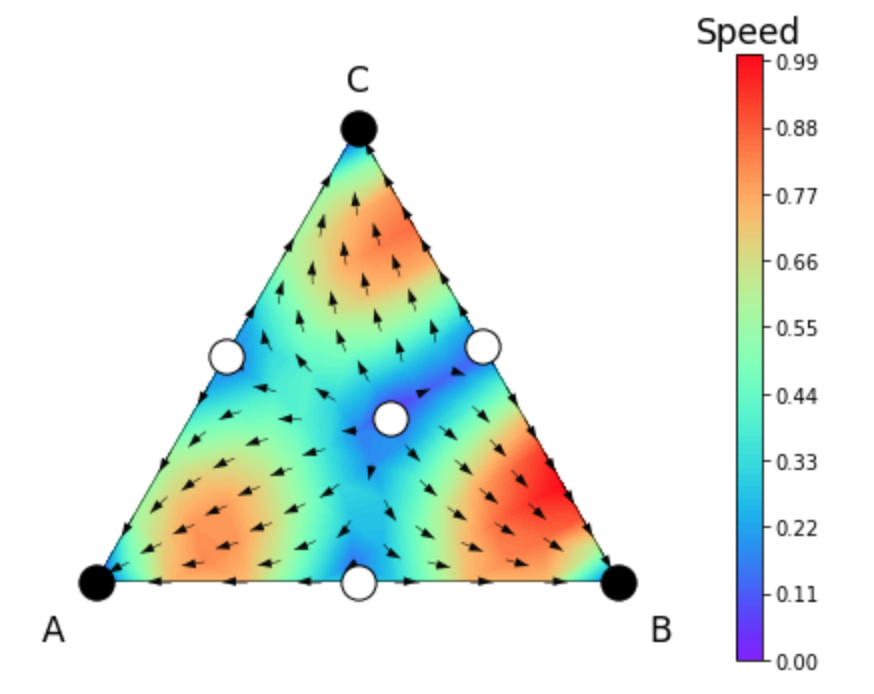}
    \caption{Directional Field of \cref{table:MetaPayoff}}
    \label{fig:directional_example}
\end{figure}

\begin{figure}[!h]
    \centering
    \includegraphics[width=5cm]{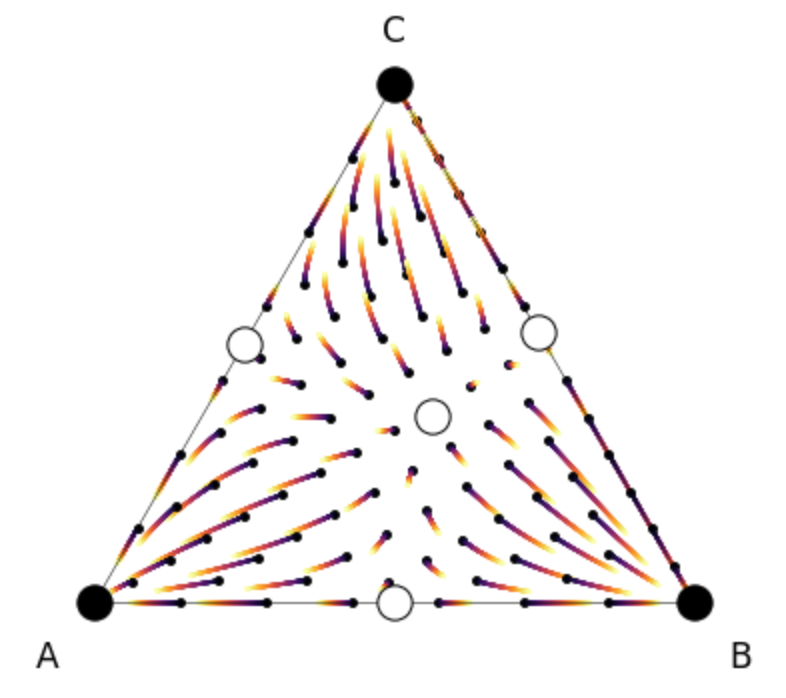}
    \caption{Trajectory Plot of \cref{table:MetaPayoff}}
    \label{fig:trajectory_example}
\end{figure}

Another example of a 2-player meta-game payoff table of 3 strategies is in \cref{table:MetaPayoff} with its corresponding directional field as shown in \cref{fig:directional_example} and its trajectory plot in \cref{fig:trajectory_example}, where the white circles denote unstable equilibria (saddle points) and black solid circles denote globally stable equilibria.

\subsection{Proof of \cref{thm:approx_equili_riskaversegame}}\label{sec:approx_equili_proof}

\begin{theorem}
For a Normal Form Game with $p$ players, and each player $i$ chooses a strategy $\pi^i$ from a set of strategies $S^i = \{\pi^i_1, ..., \pi^i_k\}$ and receives a risk averse payoff $h^i(\pi^1, ..., \pi^p):S^1\times...\times S^p\rightarrow\sR$ satisfying \cref{ass:stochastic_reward_bounded}. If $\mathbf{x}$ is a Nash Equilibrium for the game $\hat{h}^i (\pi^1, ..., \pi^p)$, then it is a $2\epsilon$-Nash equilibrium for the game $h^i (\pi^1, ..., \pi^p)$ with probability $1-\delta$ if we play the game for  $n$ times, where 
\begin{align}
     n  \ge \max
     \left\{ -\frac{8R^2}{\epsilon^2}\log\left[\frac{1}{4}\left(1-(1-\delta)^{\frac{1}{|S^1|\times ...\times |S^p|\times p}}\right)\right], \right.
     \left. \frac{64\beta^2\omega^2\cdot\Gamma(2)}{\epsilon^2\left[1-(1-\delta)^{\frac{1}{|S^1|\times...\times |S^p|\times p}}\right]}\right\}
\end{align}
\end{theorem}
\begin{assumption}\label{ass:stochastic_reward_bounded}
The stochastic return $h$ 
(for each player and each strategy) for each simulation has a sub-Gaussian tail. i,e, there exists $\omega > 0$ s.t. \begin{align}
    \expectation\left[exp\left(c\cdot(h-\expectation[h])\right)\right]&\leq exp\left(\frac{\omega^2 c^2}{2}\right) \qquad \forall c\in \sR
\end{align}
And we also select $R>0$ s.t. $h\in[-R, R]$ almost surely.
\end{assumption}
\begin{proof}
Note that we have the following relation: \begin{align}\label{eq:Equilibrium_Approximation_1}
\mathbb{E}_{\pi\sim \mathbf{x}}\left[h^i (\pi)\right] = \mathbb{E}_{\pi\sim \mathbf{x}}\left[\hat{h}^i (\pi)\right] + \mathbb{E}_{\pi\sim \mathbf{x}}\left[h^i (\pi) - \hat{h}^i (\pi)\right]
\end{align}

Then 
\begin{align}
&\mathbb{E}_{\pi^{-i}\sim\mathbf{x}^{-i}}\left[h^{i}(\pi^i, \mathbf{\pi}^{-i})\right] = \mathbb{E}_{\pi^{-i}\sim\mathbf{x}^{-i}}\left[\hat{h}^{i}(\pi^i, \mathbf{\pi}^{-i})\right] + \mathbb{E}_{\pi^{-i}\sim\mathbf{x}^{-i}}\left[h^{i}(\pi^i, \mathbf{\pi}^{-i}) - \hat{h}^{i}(\pi^i, \mathbf{\pi}^{-i})\right]\\
&\underset{\pi^i }{\max}~\mathbb{E}_{\pi^{-i}\sim\mathbf{x}^{-i}}\left[h^{i}(\pi^i, \mathbf{\pi}^{-i})\right] \le \underset{\pi^i }{\max}~\mathbb{E}_{\pi^{-i}\sim\mathbf{x}^{-i}}\left[\hat{h}^{i}(\pi^i, \mathbf{\pi}^{-i})\right] +\underset{\pi^i }{\max}~\mathbb{E}_{\pi^{-i}\sim\mathbf{x}^{-i}}\left[h^{i}(\pi^i, \mathbf{\pi}^{-i}) - \hat{h}^{i}(\pi^i, \mathbf{\pi}^{-i})\right]
\end{align}

Hence, 
\begin{align}\label{eq:Equilibrium_Approximation_2}
&\underset{\pi^i }{\max}~\mathbb{E}_{\pi^{-i}\sim\mathbf{x}^{-i}}\left[h^{i}(\pi^i, \mathbf{\pi}^{-i})\right] - \mathbb{E}_{\pi\sim\mathbf{x}}\left[h^i (\pi)\right]\\
\le & \underbrace{\underset{\pi^i }{\max}\;\mathbb{E}_{\pi^{-i}\sim\mathbf{x}^{-i}}\left[\hat{h}^{i}(\pi^i, \mathbf{\pi}^{-i})\right] - \mathbb{E}_{\pi\sim\mathbf{x}}\left[\hat{h}^i (\pi)\right]}_{=0 \text{ since } \textbf{x} \text{ is a Nash Equilibrium for }\hat{h}^{i}} + \underbrace{\underset{\pi^i }{\max}\;\mathbb{E}_{\pi^{-i}\sim\mathbf{x}^{-i}}\left[h^{i}(\pi^i, \mathbf{\pi}^{-i}) - \hat{h}^{i}(\pi^i, \mathbf{\pi}^{-i})\right]}_{\le \epsilon} + \underbrace{\mathbb{E}_{\pi\sim\mathbf{x}}\left[\hat{h}^{i}(\pi) - h^{i}(\pi)\right]}_{\le\epsilon}
\end{align}

Hence, if we can control the difference between $|h^i (\pi)-\hat{h}^{i}(\pi)|$ uniformly over players and actions, then an equilibrium for the empirical game is almost an equilibrium for the game defined by the reward function. Hence the question is how many samples $n$ do we need to assess that a Nash equilibrium for $\hat{h}$ is a $2\epsilon$-Nash equilibrium for $h$ for a fixed confidence $\delta$ and a fixed $\epsilon$.

In the following, in short, we fix player $i$ and the joint strategy $\pi = (\pi^1,..., \pi^p)$ for $p$ players and and in short, denote $h^i = h^i(\pi)$,  $\hat{h}^i = \hat{h}^i(\pi)$.
By Hoeffding inequality, 
\begin{align}\label{eq:bound_hoeffding}
    \mathbb{P}\left[\left|\bar{R^i} -\mathbb{E}[R^i] \right|\geq \frac{\epsilon}{2}\right]\leq 2\cdot exp\left(-\frac{\epsilon^2 n}{8R^2}\right)
\end{align}

Now, it remains to give a batch scenario for the unbiased estimator of variance penalty term. Denote 
$V^2_n = \frac{1}{n-1}\sum_{j=1}^{n}\left(R^i_j - \bar{R^i}\right)^2$, then $\mathbb{E}[V^2_n] = \mathbb{V}ar[R^i] = \sigma^2$, i.e., it's an unbiased estimator of the game variance. We first compute the variance of $V^2_n$.

Let $Z^i_j = R^i_j - \expectation[R^i]$, then $\expectation[Z^i] = 0$ and $Z^i_1, ...Z^i_n$ are independent. Then we have
\begin{align}
    V^2_n = \mathbb{V}ar[R^i] = \mathbb{V}ar[Z^i].
\end{align} 
\begin{align}\label{eq:variance_of_samplevariance}
    &\mathbb{V}ar[V^2_n] = \mathbb{E}[V^4_n] - (\mathbb{E}[V^2_n])^2\\
    &= \expectation\left[\frac{n^2(\sum_{j=1}^{n}(Z_j^i)^2)^2 - 2n(\sum_{j=1}^{n}(Z_j^i)^2) (\sum_{j=1}^{n}Z_j^i)^2 + (\sum_{j=1}^{n}Z_j^i)^4}{n^2(n-1)^2}\right] - \sigma^4\\
    &= \frac{n^2\expectation\left[\left(\sum_{j=1}^n(Z_j^i)^2\right)^2\right] - 2n\expectation\left[\left(\sum_{j=1}^n (Z_j^i)^2\right)\left(\sum_{j=1}^n Z_j^i\right)^2\right] + \expectation\left[\left(\sum_{j=1}^n Z_j^i\right)^4\right]}{n^2(n-1)^2} - \sigma^4
\end{align}
Since $Z_1^i, ..., Z_n^i$ are independent, then we have that for distinct $j,k,m$, \begin{align}
    \expectation[Z^i_j Z^i_k] = 0; \quad \expectation[(Z^i_j)^3 Z^i_k] = 0;\quad \expectation[(Z^i_j)^2 Z^i_k Z^i_m] = 0.
\end{align} 
And we denote\begin{align}
    \expectation[(Z^i_j)^2 (Z^i_k)^2] = \mu_2^2 = \sigma^4;\quad \expectation[(Z^i_j)^4] =\mu_4.
\end{align}
Then, with algebraic manipulations, we can simplify \cref{eq:variance_of_samplevariance} as:\begin{align}
    \mathbb{V}ar[V_n^2] &= \frac{n^2\left(n\mu_4 + n(n-1)\mu_2^2\right) - 2n(n\mu_4 + n(n-1)\mu_2^2) + n\mu_4 + 3n(n-1)\mu_2^2}{n^2(n-1)^2} - \sigma^4\\
    &= \frac{(n-1)\mu_4 +(n^2-2n+3)\sigma^4}{n(n-1)} - \sigma^4\\
    &= \frac{\mu_4}{n} - \frac{\sigma^4 (n-3)}{n(n-1)}.
\end{align}
By Chebyshev's inequality, \begin{align}
    \mathbb{P}\left[\left|V_n^2 - \mathbb{V}ar[R^i]\right|\geq \frac{\epsilon}{2\beta}\right]&\leq \frac{\mathbb{V}ar[V_n^2]}{(\frac{\epsilon}{2\beta})^2}\\
    &\leq \frac{ 4\beta^2\left(\frac{\mu_4}{n} - \frac{\sigma^4 (n-3)}{n(n-1)}\right)}{\epsilon^2}
\end{align}
By \cref{ass:stochastic_reward_bounded}, \begin{align}
    \mu_4\leq 16\omega^2\cdot\Gamma(2)
\end{align}

By triangle inequality, \begin{align}
    \sP\left[\left|h^i - \hat{h}^i\right|\geq \epsilon\right]&\le \sP\left[\left|\expectation[R^i]-\bar{R}^i\right|+\beta\cdot\left|V_n^2 - \mathbb{V}ar[R^i]\right|\geq \epsilon\right]\\
    &\le \sP\left[\left|\expectation[R^i]-\bar{R}^i\right|\ge \frac{\epsilon}{2}\;or\;\beta\cdot\left|V_n^2 - \mathbb{V}ar[R^i]\right|\geq \frac{\epsilon}{2}\right]\\
    &\le \sP\left[\left|\expectation[R^i]-\bar{R}^i\right|\ge \frac{\epsilon}{2}\right] + \sP\left[\left|V_n^2 - \mathbb{V}ar[R^i]\right|\geq \frac{\epsilon}{2\beta}\right]\\
    &\le 2\cdot exp\left(-\frac{\epsilon^2 n}{8R^2}\right) + \frac{ 4\beta^2\left(\frac{16\omega^2\cdot\Gamma(2)}{n} - \frac{\sigma^4 (n-3)}{n(n-1)}\right)}{\epsilon^2}\\
    &\le 2\cdot exp\left(-\frac{\epsilon^2 n}{8R^2}\right) + \frac{64\beta^2\omega^2\cdot\Gamma(2)}{n\epsilon^2}\\
    &= f(n, \epsilon).
\end{align}
Hence, for per joint strategies $\pi$ and per player $i$, we have the following bound : \begin{align}
    \sP\left[\underset{\pi,i}{\sup}\left|h^i(\pi) - \hat{h}^i (\pi)\right|<\epsilon\right]\geq \left(1-f(n,\epsilon)\right)^{|S^1|\times ...\times |S^p|\times p}
\end{align}

Hence, for \begin{align}
    n\ge \max\left\{-\frac{8R^2}{\epsilon^2}\log\left[\frac{1}{4}\left(1-(1-\delta)^{\frac{1}{|S^1|\times ...\times |S^p|\times p}}\right)\right]\;; \;\frac{64\beta^2\omega^2\cdot\Gamma(2)}{\epsilon^2\left[1-(1-\delta)^{\frac{1}{|S^1|\times...\times |S^p|\times p}}\right]}\right\}
\end{align}
we have $\sP\left[\underset{\pi,i}{\sup}\left|h^i(\pi) - \hat{h}^i (\pi)\right|<\epsilon\right]\ge 1-\delta$.

Plugging the result into \cref{eq:Equilibrium_Approximation_2}, we have \begin{align}\label{eq:Equilibrium_Approximation_3}
    &\underset{\pi^i }{\max}~\mathbb{E}_{\pi^{-i}\sim\mathbf{x}^{-i}}\left[h^{i}(\pi^i, \mathbf{\pi}^{-i})\right] - \mathbb{E}_{\pi\sim\mathbf{x}}\left[h^i (\pi)\right]\le 2\epsilon
\end{align}
\end{proof}

%% file: main.bbl
\begin{thebibliography}{38}
\providecommand{\natexlab}[1]{#1}
\providecommand{\url}[1]{\texttt{#1}}
\expandafter\ifx\csname urlstyle\endcsname\relax
  \providecommand{\doi}[1]{doi: #1}\else
  \providecommand{\doi}{doi: \begingroup \urlstyle{rm}\Url}\fi

\bibitem[Anschel et~al.(2017)Anschel, Baram, and Shimkin]{anschel2017averaged}
Anschel, O., Baram, N., and Shimkin, N.
\newblock Averaged-dqn: Variance reduction and stabilization for deep
  reinforcement learning.
\newblock In \emph{International Conference on Machine Learning}, pp.\
  176--185. PMLR, 2017.

\bibitem[Bagherzadeh et~al.(2020)Bagherzadeh, Kahani, and
  Briand]{bagherzadeh2020reinforcement}
Bagherzadeh, M., Kahani, N., and Briand, L.
\newblock Reinforcement learning for test case prioritization.
\newblock \emph{arXiv preprint arXiv:2011.01834}, 2020.

\bibitem[Balch et~al.(2019)Balch, Mahfouz, Lockhart, Hybinette, and
  Byrd]{balch2019evaluate}
Balch, T.~H., Mahfouz, M., Lockhart, J., Hybinette, M., and Byrd, D.
\newblock How to evaluate trading strategies: Single agent market replay or
  multiple agent interactive simulation?
\newblock \emph{arXiv preprint arXiv:1906.12010}, 2019.

\bibitem[Bellemare et~al.(2020)Bellemare, Candido, Castro, Gong, Machado,
  Moitra, Ponda, and Wang]{bellemare2020autonomous}
Bellemare, M.~G., Candido, S., Castro, P.~S., Gong, J., Machado, M.~C., Moitra,
  S., Ponda, S.~S., and Wang, Z.
\newblock Autonomous navigation of stratospheric balloons using reinforcement
  learning.
\newblock \emph{Nature}, 588\penalty0 (7836):\penalty0 77--82, 2020.

\bibitem[Berner et~al.(2019)Berner, Brockman, Chan, Cheung, Debiak, Dennison,
  Farhi, Fischer, Hashme, Hesse, et~al.]{berner2019dota}
Berner, C., Brockman, G., Chan, B., Cheung, V., Debiak, P., Dennison, C.,
  Farhi, D., Fischer, Q., Hashme, S., Hesse, C., et~al.
\newblock Dota 2 with large scale deep reinforcement learning.
\newblock \emph{arXiv preprint arXiv:1912.06680}, 2019.

\bibitem[Bertsekas(2009)]{Bertsekas2009}
Bertsekas, D.~P.
\newblock \emph{Neuro-dynamic programmingNeuro-Dynamic Programming}, pp.\
  2555--2560.
\newblock Springer US, Boston, MA, 2009.
\newblock ISBN 978-0-387-74759-0.
\newblock \doi{10.1007/978-0-387-74759-0_440}.
\newblock URL \url{https://doi.org/10.1007/978-0-387-74759-0_440}.

\bibitem[Bloembergen et~al.(2015)Bloembergen, Hennes, McBurney, and
  Tuyls]{bloembergen2015trading}
Bloembergen, D., Hennes, D., McBurney, P., and Tuyls, K.
\newblock Trading in markets with noisy information: An evolutionary analysis.
\newblock \emph{Connection Science}, 27\penalty0 (3):\penalty0 253--268, 2015.

\bibitem[Bowling(2000)]{bowling2000convergence}
Bowling, M.
\newblock Convergence problems of general-sum multiagent reinforcement
  learning.
\newblock In \emph{ICML}, pp.\  89--94, 2000.

\bibitem[Bowling \& Veloso(2002)Bowling and Veloso]{bowling2002multiagent}
Bowling, M. and Veloso, M.
\newblock Multiagent learning using a variable learning rate.
\newblock \emph{Artificial Intelligence}, 136\penalty0 (2):\penalty0 215--250,
  2002.

\bibitem[Buldygin \& Kozachenko(1980)Buldygin and Kozachenko]{SubGaussian80}
Buldygin, V.~V. and Kozachenko, Y.~V.
\newblock Sub-gaussian random variables.
\newblock \emph{Ukrainian Mathematical Journal}, 32\penalty0 (6):\penalty0
  483--489, 1980.
\newblock \doi{10.1007/BF01087176}.

\bibitem[Byrd et~al.(2019)Byrd, Hybinette, and Balch]{byrd2019abides}
Byrd, D., Hybinette, M., and Balch, T.~H.
\newblock Abides: Towards high-fidelity market simulation for ai research.
\newblock \emph{arXiv preprint arXiv:1904.12066}, 2019.

\bibitem[Di~Castro et~al.(2012)Di~Castro, Tamar, and Mannor]{di2012policy}
Di~Castro, D., Tamar, A., and Mannor, S.
\newblock Policy gradients with variance related risk criteria.
\newblock \emph{arXiv preprint arXiv:1206.6404}, 2012.

\bibitem[Filar \& Vrieze(1997)Filar and Vrieze]{filar-competitiveMDP}
Filar, J. and Vrieze, K.
\newblock \emph{Competitive Markov Decision Processes}.
\newblock Springer, 1997.

\bibitem[Garc{\i}a \& Fern{\'a}ndez(2015)Garc{\i}a and
  Fern{\'a}ndez]{garcia2015comprehensive}
Garc{\i}a, J. and Fern{\'a}ndez, F.
\newblock A comprehensive survey on safe reinforcement learning.
\newblock \emph{Journal of Machine Learning Research}, 16\penalty0
  (1):\penalty0 1437--1480, 2015.

\bibitem[Henderson et~al.(2018)Henderson, Islam, Bachman, Pineau, Precup, and
  Meger]{henderson2018deep}
Henderson, P., Islam, R., Bachman, P., Pineau, J., Precup, D., and Meger, D.
\newblock Deep reinforcement learning that matters.
\newblock In \emph{Proceedings of the AAAI Conference on Artificial
  Intelligence}, volume~32, 2018.

\bibitem[Hu \& Wellman(1998)Hu and Wellman]{MultiAgentQLearning98}
Hu, J. and Wellman, M.~P.
\newblock Multiagent reinforcement learning: Theoretical framework and an
  algorithm.
\newblock In \emph{Proceedings of the Fifteenth International Conference on
  Machine Learning}, ICML '98, pp.\  242–250, San Francisco, CA, USA, 1998.
  Morgan Kaufmann Publishers Inc.
\newblock ISBN 1558605568.

\bibitem[Johnson \& Zhang(2013)Johnson and Zhang]{NIPS2013_ac1dd209}
Johnson, R. and Zhang, T.
\newblock Accelerating stochastic gradient descent using predictive variance
  reduction.
\newblock In Burges, C. J.~C., Bottou, L., Welling, M., Ghahramani, Z., and
  Weinberger, K.~Q. (eds.), \emph{Advances in Neural Information Processing
  Systems}, volume~26. Curran Associates, Inc., 2013.

\bibitem[Li(2017)]{li2017deep}
Li, Y.
\newblock Deep reinforcement learning: An overview.
\newblock \emph{arXiv preprint arXiv:1701.07274}, 2017.

\bibitem[Li et~al.(2009)Li, Szepesvari, and Schuurmans]{li2009learning}
Li, Y., Szepesvari, C., and Schuurmans, D.
\newblock Learning exercise policies for american options.
\newblock In \emph{Artificial Intelligence and Statistics}, pp.\  352--359.
  PMLR, 2009.

\bibitem[Littman(2001)]{littman2001value}
Littman, M.~L.
\newblock Value-function reinforcement learning in markov games.
\newblock \emph{Cognitive systems research}, 2\penalty0 (1):\penalty0 55--66,
  2001.

\bibitem[Mihatsch \& Neuneier(2002)Mihatsch and Neuneier]{mihatsch2002risk}
Mihatsch, O. and Neuneier, R.
\newblock Risk-sensitive reinforcement learning.
\newblock \emph{Machine learning}, 49\penalty0 (2):\penalty0 267--290, 2002.

\bibitem[Mirhoseini et~al.(2020)Mirhoseini, Goldie, Yazgan, Jiang, Songhori,
  Wang, Lee, Johnson, Pathak, Bae, et~al.]{mirhoseini2020chip}
Mirhoseini, A., Goldie, A., Yazgan, M., Jiang, J., Songhori, E., Wang, S., Lee,
  Y.-J., Johnson, E., Pathak, O., Bae, S., et~al.
\newblock Chip placement with deep reinforcement learning.
\newblock \emph{arXiv preprint arXiv:2004.10746}, 2020.

\bibitem[Morimoto \& Doya(2005)Morimoto and Doya]{morimoto2005robust}
Morimoto, J. and Doya, K.
\newblock Robust reinforcement learning.
\newblock \emph{Neural computation}, 17\penalty0 (2):\penalty0 335--359, 2005.

\bibitem[Ning et~al.(2018)Ning, Lin, and Jaimungal]{ning2018double}
Ning, B., Lin, F. H.~T., and Jaimungal, S.
\newblock Double deep q-learning for optimal execution.
\newblock \emph{arXiv preprint arXiv:1812.06600}, 2018.

\bibitem[Pan et~al.(2019)Pan, Seita, Gao, and Canny]{pan2019risk}
Pan, X., Seita, D., Gao, Y., and Canny, J.
\newblock Risk averse robust adversarial reinforcement learning.
\newblock In \emph{2019 International Conference on Robotics and Automation
  (ICRA)}, pp.\  8522--8528. IEEE, 2019.

\bibitem[Pinto et~al.(2017)Pinto, Davidson, Sukthankar, and
  Gupta]{pinto2017robust}
Pinto, L., Davidson, J., Sukthankar, R., and Gupta, A.
\newblock Robust adversarial reinforcement learning.
\newblock In \emph{International Conference on Machine Learning}, pp.\
  2817--2826. PMLR, 2017.

\bibitem[Sharpe(1994)]{sharpe1994sharpe}
Sharpe, W.~F.
\newblock The sharpe ratio.
\newblock \emph{Journal of portfolio management}, 21\penalty0 (1):\penalty0
  49--58, 1994.

\bibitem[Shen et~al.(2014)Shen, Tobia, Sommer, and Obermayer]{shen2014risk}
Shen, Y., Tobia, M.~J., Sommer, T., and Obermayer, K.
\newblock Risk-sensitive reinforcement learning.
\newblock \emph{Neural computation}, 26\penalty0 (7):\penalty0 1298--1328,
  2014.

\bibitem[Spooner et~al.(2018)Spooner, Fearnley, Savani, and
  Koukorinis]{spooner2018market}
Spooner, T., Fearnley, J., Savani, R., and Koukorinis, A.
\newblock Market making via reinforcement learning.
\newblock \emph{arXiv preprint arXiv:1804.04216}, 2018.

\bibitem[Szepesvári \& Littman(1999)Szepesvári and Littman]{ValueBasedRL99}
Szepesvári, C. and Littman, M.
\newblock A unified analysis of value-function-based reinforcement-learning
  algorithms.
\newblock \emph{Neural computation}, 11:\penalty0 2017--59, 12 1999.
\newblock \doi{10.1162/089976699300016070}.

\bibitem[Th{\'e}ate \& Ernst(2021)Th{\'e}ate and Ernst]{theate2021application}
Th{\'e}ate, T. and Ernst, D.
\newblock An application of deep reinforcement learning to algorithmic trading.
\newblock \emph{Expert Systems with Applications}, 173:\penalty0 114632, 2021.

\bibitem[Tobia et~al.(2013)Tobia, Guo, Schwarze, Böhmer, Gläscher, Finckh,
  Marschner, Büchel, Obermayer, and Sommer]{RiskSensitiveshen13}
Tobia, M., Guo, R., Schwarze, U., Böhmer, W., Gläscher, J., Finckh, B.,
  Marschner, A., Büchel, C., Obermayer, K., and Sommer, T.
\newblock Neural systems for choice and valuation with counterfactual learning
  signals.
\newblock \emph{NeuroImage}, 89, 12 2013.
\newblock \doi{10.1016/j.neuroimage.2013.11.051}.

\bibitem[Tuyls et~al.(2020)Tuyls, Perolat, Lanctot, Hughes, Everett, Leibo,
  Szepesv{\'a}ri, and Graepel]{tuyls2020bounds}
Tuyls, K., Perolat, J., Lanctot, M., Hughes, E., Everett, R., Leibo, J.~Z.,
  Szepesv{\'a}ri, C., and Graepel, T.
\newblock Bounds and dynamics for empirical game theoretic analysis.
\newblock \emph{Autonomous Agents and Multi-Agent Systems}, 34\penalty0
  (1):\penalty0 1--30, 2020.

\bibitem[Wainwright(2019)]{wainwright2019variance}
Wainwright, M.~J.
\newblock Variance-reduced $ q $-learning is minimax optimal.
\newblock \emph{arXiv preprint arXiv:1906.04697}, 2019.

\bibitem[Walsh et~al.(2002)Walsh, Das, Tesauro, and
  Kephart]{walsh2002analyzing}
Walsh, W.~E., Das, R., Tesauro, G., and Kephart, J.~O.
\newblock Analyzing complex strategic interactions in multi-agent systems.
\newblock In \emph{AAAI-02 Workshop on Game-Theoretic and Decision-Theoretic
  Agents}, pp.\  109--118, 2002.

\bibitem[Weibull(1997)]{weibull1997evolutionary}
Weibull, J.~W.
\newblock \emph{Evolutionary game theory}.
\newblock MIT press, 1997.

\bibitem[Weinberg \& Rosenschein(2004)Weinberg and
  Rosenschein]{weinberg2004best}
Weinberg, M. and Rosenschein, J.~S.
\newblock Best-response multiagent learning in non-stationary environments.
\newblock In \emph{Proceedings of the Third International Joint Conference on
  Autonomous Agents and Multiagent Systems-Volume 2}, pp.\  506--513, 2004.

\bibitem[Wellman(2006)]{wellman2006methods}
Wellman, M.~P.
\newblock Methods for empirical game-theoretic analysis.
\newblock In \emph{AAAI}, pp.\  1552--1556, 2006.

\end{thebibliography}
